\documentclass[11pt]{article}

\usepackage[final]{acl}
\usepackage{booktabs}
\usepackage{bbding}
\usepackage{pifont}
\usepackage{times}
\usepackage{amsthm}
\newtheorem{proposition}{Proposition}
\usepackage{amsmath}
\usepackage{latexsym}
\usepackage{lipsum}
\usepackage{arydshln} 
\usepackage[dvipsnames]{xcolor}
\usepackage{amsfonts}
\usepackage{colortbl}
\usepackage{booktabs}
\usepackage[T1]{fontenc}
\definecolor{col_color}{HTML}{CEEDFB}
\definecolor{row_color}{HTML}{DEE6FC}
\usepackage[utf8]{inputenc}

\usepackage{microtype}
\usepackage{multicol}
\usepackage{multirow}
\usepackage{enumitem}
\usepackage{inconsolata}

\usepackage{graphicx}
\usepackage{xspace}

\newcommand{\modelname}{HAPO\xspace}
\title{Where Hindsight Credit Can Reside: A Signed-Capacity View of Token Updates in RLVR}

\author{Yuhang He$^{1}$, Haodong Wu$^{1}$, Siyi Liu$^{1}$, Hongyu Ge$^{1}$, Hange Zhou$^{1}$, Keyi Wu$^{1}$ \\
{\bf  Zhuo Zheng$^{2}$, Qihong Lin$^{2}$, Zixin Zhong$^{1}$, Yongqi Zhang$^{1}$\thanks{\ \ Corresponding Authors}} \\
  $^{1}$ The Hong Kong University of Science and Technology (Guangzhou) \\
  $^{2}$ Huawei Technologies Ltd. \\
  {\tt \{yhe534,hwu315\}@connect.hkust-gz.edu.cn, yongqizhang@hkust-gz.edu.cn}\\
}

\begin{document}
\maketitle
\begin{abstract}
  Reinforcement Learning with Verifiable Rewards (RLVR) improves the reasoning ability of Large Language Models (LLMs), but sparse outcome rewards make token-level credit assignment difficult.
  We study token-level credit as a reward-conditioned shift from the behavior policy to a hindsight posterior.
  In autoregressive RLVR, this shift can be expressed through Conditional Mutual Information (CMI), which shows that token entropy upper-bounds possible hindsight credit.
  Entropy, however, indicates capacity rather than update direction, so we introduce the Four Quadrant Decomposition to separate updates by reward polarity and token entropy.
  Controlled interventions show that these two factors jointly shape token updates.
  Sustained reasoning gains concentrate in signed high-entropy quadrants, whereas low-entropy updates saturate quickly.
  Based on this analysis, we propose Hindsight-Aware Policy Optimization (\modelname), a sign-preserving modification to GRPO that performs capacity-guided advantage reallocation.
  Experiments on mathematical reasoning benchmarks in two model settings show that \modelname achieves competitive performance among entropy-aware baselines.
\end{abstract}
\section{Introduction}
Reinforcement Learning with Verifiable Rewards (RLVR) improves reasoning in Large Language
Models (LLMs) by optimizing against outcome rewards from
self-generated responses~\cite{Guo2025Deepseek,lambert2024tulu,DBLP:journals/corr/abs-2110-14168,
DBLP:conf/nips/HendrycksBKABTS21,DBLP:conf/iclr/JimenezYWYPPN24},
bypassing the need for dense step-level supervision~\cite{DBLP:conf/icml/ChuZYTXSLL025}.
In practice, methods such as
GRPO~\cite{DBLP:journals/corr/abs-2402-03300} broadcast a single sequence-level reward
to every token in a trajectory, giving rise to a credit assignment problem:
a trajectory's outcome may hinge on a few pivotal decisions, yet most tokens are routine continuations that receive identical credit~\cite{Minsky1995StepsTA,DBLP:conf/iclr/LightmanKBEBLLS24}. Thus, RLVR requires a token-level view of assigning outcome feedback.

\begin{figure*}[t]
  \centering
  \includegraphics[width=1\linewidth]{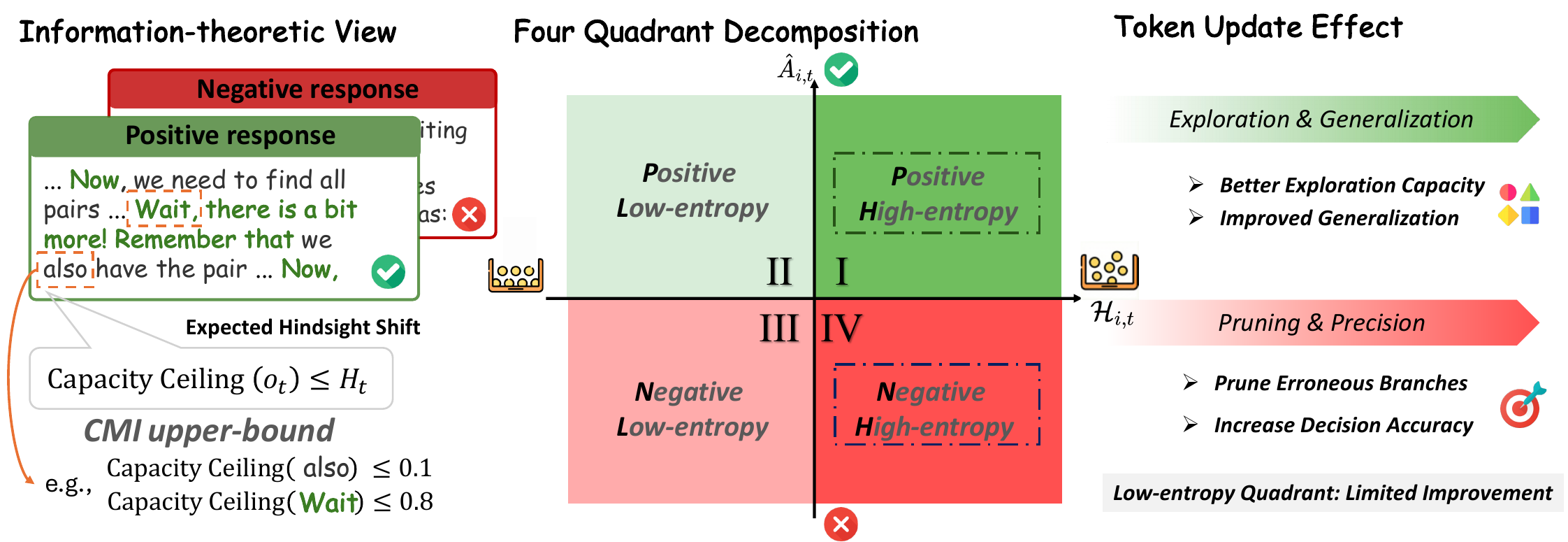}
  \caption{Tokens in \textbf{Bold} exhibit high entropy. From an information-theoretic view, high-entropy and low-entropy tokens differ in how much outcome information they can carry.
  We decompose token updates in RLVR into four quadrants to study the mechanism of token-level credit assignment. See Appendix~\ref{app:example_output} for full entropy examples.}
  \label{fig:intro_pic}
\end{figure*}

Recent work studies two observable axes for this problem: \textbf{reward polarity}, which redistributes credit by positive and negative samples to decide whether sampled tokens should be reinforced or discouraged~\cite{DBLP:conf/nips/YuanYTWHH23,Zhu2025TheSE}, and \textbf{token entropy}, which identifies uncertain tokens where the policy has more room to change~\cite{DBLP:journals/corr/abs-2506-01939,DBLP:journals/corr/abs-2506-14758}.
These views are useful but incomplete.
Polarity gives the update direction but lacks within-trajectory selectivity, whereas entropy is the reverse. In other words, high-entropy tokens in successful and failed trajectories should receive different corrections.
This leaves a sharper question---\textit{given the outcome, how to adjust earlier token probabilities relative to the behavior policy?}

We answer this from a \textbf{hindsight} perspective since the reward is observed only after generation.
\textbf{Hindsight Credit Assignment} (HCA) formalizes the ideal learning signal as a reward-conditioned posterior-prior shift: increase the probability of a token in proportion to how much conditioning on the final outcome makes it more likely, and decrease it if vice versa~\cite{Harutyunyan2019Hindsight,andrychowicz2017hindsight}.
Directly estimating this hindsight posterior is infeasible in autoregressive LLMs without counterfactual rollouts or an auxiliary posterior model.
We therefore use HCA as an analytical target. Under this view, reward polarity provides a coarse correction sign, and entropy indicates where the behavior policy has enough distributional capacity for a non-trivial correction.

To make this target analyzable, we connect hindsight credit to \textbf{Conditional Mutual Information} (CMI).
The CMI between the next-token and reward characterizes the expected posterior-prior shift.
Still this quantity is intractable, but it is upper-bounded by token entropy.
Entropy thus acts as a computable capacity ceiling: low-entropy tokens can support only small expected hindsight shifts, while high-entropy tokens flag candidate positions for larger distribution shifts.
Because this bound is one-sided, we then propose the \textbf{Four Quadrant Decomposition} as a diagnostic framework that factorizes token updates along reward polarity and entropy (Figure~\ref{fig:intro_pic}).
It tests quadrant-isolated variants to locate useful update effects predicted by this hindsight-capacity view.
The interventions show that low-entropy updates can help early optimization but quickly saturate, whereas signed high-entropy updates produce more persistent gains by consolidating successful uncertain branches and pruning failed ones.

These diagnostics motivate \textbf{Hindsight-Aware Policy Optimization} (\textbf{HAPO}), a lightweight GRPO modification that emphasizes capacity-guided credit assignment.
In effect, HAPO down-weights low-entropy updates and redirects uniform outcome toward signed high-entropy positions, where hindsight credit can plausibly reside.
On math reasoning and out-of-domain benchmarks across two models, \modelname achieves strong aggregate performance and competitive per-benchmark results.
Our contributions are summarized as follows:

\begin{itemize}[leftmargin=*, nosep]
  \item We formulate token-level credit in RLVR from a hindsight perspective and instantiate it with CMI, establishing entropy as a computable capacity ceiling for possible posterior-prior shifts.

  \item We introduce the \textbf{Four Quadrant Decomposition}, a diagnostic framework that factorizes token updates by reward polarity and token entropy to identify where sustained reasoning improvements concentrate.

  \item Based on this theoretical and empirical analysis, we propose \textbf{\modelname}, a capacity-guided method that achieves competitive performance.




\end{itemize}

\section{Preliminaries}
\subsection{Reinforcement Learning with Verifiable Rewards}

Formally, we consider an LLM parameterized by $\theta$. 
Given a dataset of query prompts $Q$, standard RLVR learns a policy $\pi_\theta$ to minimize the objective:
\begin{equation}
\mathcal{L}_\mathrm{RLVR}(\theta) = - \mathbb{E}_{q \sim Q, o \sim \pi_\theta(\cdot|q)} [r(o, q)],
\end{equation}
where $r(o, q)$ is a verifiable reward that evaluates the correctness of the response $o$ given the query $q$.
To enhance training stability and efficiency, GRPO samples a group of responses $\{o_i\}_{i=1}^G$ for each query from the old behavior policy $\pi_{\theta_{\mathrm{old}}}$ and optimizes the following surrogate objective loss~\cite{DBLP:journals/corr/abs-2402-03300,DBLP:journals/corr/abs-2503-14476}:
\begin{equation}
\hspace{-0.2cm}\resizebox{0.92\hsize}{!}{$\displaystyle
\begin{aligned}
\mathcal{L}_{\mathrm{GRPO}}(\theta) ={}&- \mathbb{E}_{q\sim Q,\{o_i\}_{i=1}^G\sim\pi_{\theta_{\mathrm{old}}}(\cdot|q)} \\
&\Bigg[\!\frac{1}{\sum_{i=1}^G \!|o_i|}\!\sum_{i=1}^G\!\sum_{t=1}^{|o_i|} 
\min\!\Big(\rho_{i,t}\cdot\hat{A}_{i,t},\\
&\!\operatorname{clip}(\rho_{i,t}, \!1\!-\!\epsilon_{\text{low}},\!1\!+\!\epsilon_{\text{high}})\cdot \hat{A}_{i,t}\Big)\Bigg],
\end{aligned}$}
\label{eq:grpo}
\end{equation}
where $\hat{A}_{i,t}=\frac{r_i-\text{mean}(\{r_i\}_{i=1}^G)}{\text{std}(\{r_i\}_{i=1}^G)}$ represents the group advantage, normalizing the outcome reward $r_i$.
$\rho_{i,t}=\frac{\pi_\theta(o_{i,t}|q,o_{i,<t})}{\pi_{\theta_{\mathrm{old}}}(o_{i,t}|q,o_{i,<t})}$ denotes the importance sampling ratio at step $t$ of trajectory $o_i$. $\epsilon_\text{{high}/{low}}$ indicates the clipping range of importance ratios.
\subsection{Hindsight Credit Assignment}
\label{sec:prelim_hca}


Hindsight Credit Assignment (HCA) offers a backward view of token updates in RLVR~\cite{Harutyunyan2019Hindsight}.
At decoding step $t$, the model observes the prefix-conditioned state $s_{i,t} := (q,o_{i,<t})$ and samples the token $o_{i,t}\sim \pi_\theta(\cdot| s_{i,t})$.
HCA conceptualizes learning the reward-conditioned hindsight posterior over the next token via the posterior-prior ratio. This ratio $\rho^{\mathrm{hs}}$ is the ideal HCA target for token-level credit:
\begin{equation}
\rho^{\mathrm{hs}}
=
\frac{h_{\pi_\theta}(o_{i,t}\mid s_{i,t}, r)}
{\pi_\theta(o_{i,t}\mid s_{i,t})}.
\label{eq:hca_ratio}
\end{equation}
where $h_{\pi_\theta}$ is the hindsight posterior distribution over the next token.
If $\rho^{\mathrm{hs}}>1$, it makes the sampled token more likely than under the behavior policy $\pi_\theta$, and vice versa.
However, precise $h_{\pi_\theta}$ estimation is difficult in high-dimensional autoregressive spaces, as it would require counterfactual future rollouts for alternative next tokens from the same prefix or an additional model for the outcome-conditioned posterior~\cite{DBLP:conf/icml/PooleOOAT19,DBLP:conf/aaai/FoersterFANW18,Tan2026Hindsight}.
We therefore use HCA as an analytical target, and next reinterpret $\rho^{\mathrm{hs}}$ through CMI.

\subsection{Policy Entropy}

Policy entropy quantifies the uncertainty of the model's predictive distribution over the vocabulary $\mathcal{V}$~\cite{DBLP:conf/iclr/KuhnGF23}. 
Given a query $q$ and a partial trajectory $o_{i,<t}$, the token entropy is defined as the Shannon entropy~\cite{DBLP:journals/bstj/Shannon48a} of the output distribution at step $t$:
$\mathcal{H}_{i,t} = -\!\sum_{v \in \mathcal{V}} \pi_\theta(v \mid q, o_{i,{<t}}) \log \pi_\theta(v \mid q, o_{i,{<t}})$.
The policy entropy for a trajectory $o_i$ is approximated as the average token entropy.
Prior RLVR methods use entropy to identify uncertain or forking positions and to shape reward signals~\cite{DBLP:journals/corr/abs-2506-01939,DBLP:journals/corr/abs-2506-14758,chen2025beyond}.
While these methods reshape reward, the underlying mechanisms of credit assignment remain underexplored.

\section{Entropy-Bounded Hindsight Credit}
\label{sec:info_theory}

Because the outcome reward is observed only after generation, token-level credit in RLVR is inherently a hindsight question: how would the next-token distribution at a prefix change after conditioning on the final outcome?
Here, the HCA posterior-prior shift serves as the analytical target, and we adapt the information-theoretic credit view of~\citet{arumugam2021information} to the autoregressive LLM setting.

\subsection{HCA to CMI}
Recall from Section~\ref{sec:prelim_hca} that HCA defines the posterior-prior ratio $\rho^{\mathrm{hs}}$, which serves as the ideal but intractable pointwise hindsight credit target.
CMI makes this target analyzable by aggregating pointwise hindsight evidence into a distribution-level shift. It measures how much the next-token distribution would change once the trajectory outcome is revealed.


\begin{proposition}[HCA--CMI Connection]
\label{prop:hindsight_cmi}
For any prefix state $s_{i,t}$, let $h_{\pi_\theta}(\cdot\mid s_{i,t},r)$ denote the reward-conditioned hindsight distribution over the random next token $\textbf{o}_{i,t}$. Then
\begin{equation*}
\resizebox{\linewidth}{!}{$\displaystyle
  I(\textbf{o}_{i,t};\, r \mid s_{i,t})
  =
  \mathbb{E}
  \left[
  \mathrm{KL}\!\left(h_{\pi_\theta}(\cdot\mid s_{i,t},r)\,\|\,\pi_\theta(\cdot\mid s_{i,t})\right)
  \right].
$}
\end{equation*}
\end{proposition}

Proposition~\ref{prop:hindsight_cmi} instantiates CMI as the expected distribution-level counterpart of the HCA posterior-prior shift.
Equivalently, CMI is the expected log HCA ratio over outcomes and hindsight-conditioned token choices, so it aggregates rather than replaces the pointwise target.
Exact computation would require the hindsight distribution $h_{\pi_\theta}$, which is not directly available as discussed in Section~\ref{sec:prelim_hca}.
This connection leads to a tractable consequence:

\begin{proposition}[Entropy Capacity Bound]
\label{prop:credit_bound}
For any prefix state $s_{i,t}$, the information contribution of the next token to the outcome $r$ satisfies:
\begin{equation*}
  I(\textbf{o}_{i,t};\, r | s_{i,t}) \leq \mathcal{H}_{i,t}.
  \label{eq:cmi_bound}
\end{equation*}
\end{proposition}

That is, \textbf{$\mathcal{H}_{i,t}$ upper-bounds the token's reward-relevant information in expectation}.
In this sense, entropy is not an estimator of token credit itself, but a computable ceiling on where non-trivial hindsight credit can reside.
Low-entropy tokens (syntax, routine algebra, common suffixes) have very low $\mathcal{H}_{i,t}$ and therefore cannot support a large expected hindsight shift.
High-entropy positions, by contrast, merely mark candidate branching positions where such a shift is possible.
The full derivation and interpretation are in Appendix~\ref{app:cmi_derivation}.

\subsection{Signed-Capacity Predictions}
\label{sec:grpo_bias}

In standard GRPO, the advantage $\hat{A}_{i,t}$ is constant for all tokens within a trajectory.
This broadcasts the same outcome signal to tokens with very different hindsight capacity---routine tokens with little room for reward-conditioned shift and branching positions where such shifts are more plausible.
Proposition~\ref{prop:credit_bound} therefore motivates comparing this uniform update with the token-level geometry of reasoning trajectories.
In this view, the issue is not only that uniform assignment is coarse.
It also ignores where a reward-conditioned posterior-prior shift can plausibly reside.

Entropy alone is still unsigned: it identifies where a non-trivial shift is possible, but not whether the sampled token should be reinforced or suppressed.
The sign of the group advantage supplies this coarse reward polarity, while entropy supplies the capacity axis.
Because this condition is necessary rather than sufficient, we treat the implications below as {empirical predictions} to be probed by quadrant-restricted interventions before proposing a method.
This gives a signed-capacity hypothesis for controlled quadrant-isolated training:

\vspace{2pt}
\noindent\textbf{Prediction~1} (Low-entropy saturation). \textit{Training exclusively on low-entropy tokens delivers at most modest or short-lived gains, since they have little capacity for reward-relevant information $I(\textbf{o}_{i,t}; r | s_{i,t})$.}

\vspace{2pt}
\noindent\textbf{Prediction~2} (High-entropy candidate effects). \textit{Training on high-entropy positions is a plausible source of sustained improvements because it permits larger information $I(\textbf{o}_{i,t}; r | s_{i,t})$.}

\vspace{2pt}
\noindent\textbf{Prediction~3} (Polarity--entropy interaction). \textit{Because the hindsight distribution $h_{\pi_\theta}$ differs between successful and failed outcomes, high-entropy updates should exhibit asymmetric effects across reward polarity rather than a uniform response.}

\vspace{2pt}

\begin{figure*}[ht]
  \includegraphics[width=\linewidth]{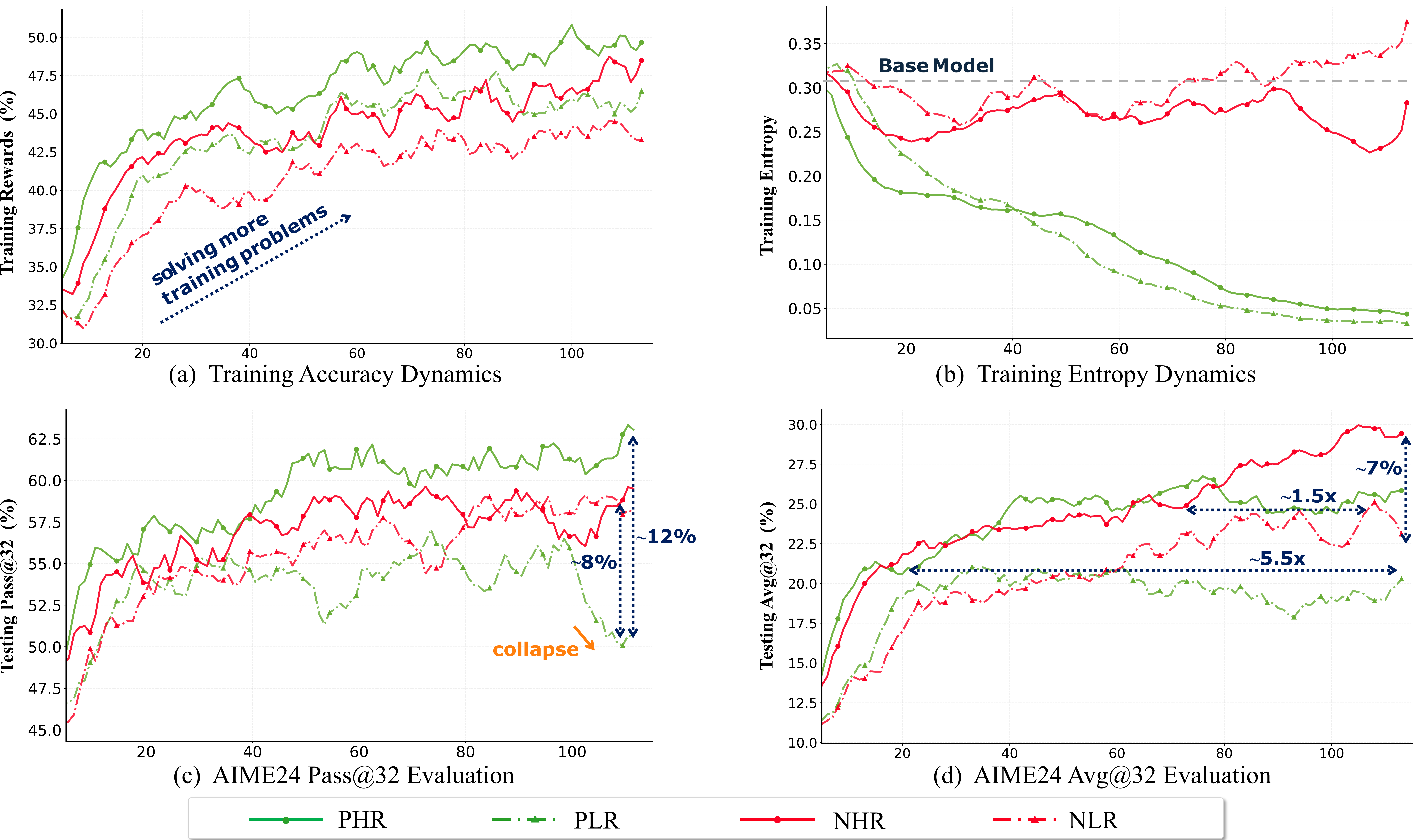}
  \caption{
  Training dynamics and evaluation for four-quadrant diagnostic interventions.  (a) Training solve rate. (b) Training entropy. (c) Pass@32 on AIME24. (d) Avg@32 on AIME24. The x-axis denotes training steps.
  }
  \label{fig:decompose}
\end{figure*}

\section{Four Quadrant Decomposition}
\label{sec:decompose_token_effect}
Section~\ref{sec:info_theory} explains why entropy constrains where hindsight credit can plausibly reside.
Controlled interventions then ask where signed GRPO updates produce durable behavioral gains.
Building on~\citet{Zhu2025TheSE}, who decompose RLVR into positive and negative paradigms based on reward polarity,
we cross reward polarity with token entropy, yielding a four-quadrant experimental framework.
By restricting gradient updates to specific quadrants, this framework yields diagnostic readouts of where signed update effects are concentrated without treating the observed metrics as direct measurements of latent token credit.

\subsection{Quadrant Setup}
Prior work links token entropy to exploratory behaviors~\cite{DBLP:journals/corr/abs-2506-01939,DBLP:journals/corr/abs-2506-14758}, but how tokens of different entropy levels influence model optimization and generalization remains underexplored.

These observations motivate us to decompose token reinforcement in GRPO into four quadrants based on two criteria:
(1) \textbf{Reward polarity (positive / negative)}: whether a token belongs to a positive or negative trajectory, determined by the outcome reward $r_i$;
(2) \textbf{Token entropy (high / low)}: whether the token entropy $\mathcal{H}_{i,t}$ is relatively higher or lower than the average entropy of the group trajectories $\{\mathcal{H}_{i,t}\}_{i=1}^G$.
In this diagnostic, entropy provides the capacity axis, while reward polarity supplies the signed correction direction.
Using these criteria, we define four diagnostic quadrants for token reinforcement as shown in Figure~\ref{fig:intro_pic}:
\begin{itemize}[leftmargin=*, nosep]
    \item \textbf{PHR}: Positive High-entropy Reinforcement,
    \item \textbf{PLR}: Positive Low-entropy Reinforcement,
    \item \textbf{NHR}: Negative High-entropy Reinforcement,
    \item \textbf{NLR}: Negative Low-entropy Reinforcement,
\end{itemize}


To investigate the underlying learning dynamics of GRPO, we formulate four corresponding sub-objectives, each restricting gradient updates exclusively to tokens within one quadrant.
The diagnostic is instantiated with \texttt{Qwen2.5-Math-7B} and evaluated on AIME24 using $\text{Avg@}32$ and $\text{Pass@}32$ as exploitation- and diversity-oriented behavioral readouts.
Appendices~\ref{app:training_hyperparameters} and~\ref{app:evaluation_setup} report the detailed training and evaluation setup.

\subsection{Findings of Optimization Dynamics}
Figure~\ref{fig:decompose} tracks training dynamics across all four quadrants.
The key diagnostic findings are:

\paragraph{Finding 1: Low-entropy updates optimize early but have limited hindsight capacity.}
PLR and NLR can accelerate early optimization because they act on segments the model is already confident about.
However, their gains saturate quickly and can even become detrimental, as seen most clearly in PLR's entropy collapse and degraded diversity in Figures~\hyperref[fig:decompose]{2c} and~\hyperref[fig:decompose]{2d}.
PLR reaches the same $\text{Avg@}32$ as PHR only after $5.5\times$ more steps and trails PHR by 12\% on $\text{Pass@}32$.
Test-time performance in Appendix~\ref{app:tt_performance} further documents this phenomenon, where PLR underperforms other quadrants and even the base model in $\text{Pass@}256$.

\paragraph{Finding 2: PHR consolidates successful high-entropy branches.}
PHR and PLR both improve training-time solve rate in Figure~\hyperref[fig:decompose]{2a}, but they diverge at evaluation in Figure~\hyperref[fig:decompose]{2c-2d}.
PHR continues to improve both $\text{Avg@}32$ and $\text{Pass@}32$, indicating that successful high-entropy branches offer reinforcing hindsight evidence that transfers beyond the sampled rollouts.
Under the capacity view, these positions are not credited merely because they are uncertain; positive reward polarity supplies the direction, while high entropy leaves enough room for a non-trivial posterior-prior shift.

\paragraph{Finding 3: NHR prunes failed high-entropy branches.}
Both negative quadrants preserve more entropy than positive-quadrant training and improve evaluation metrics, consistent with prior observations that negative reinforcement can be effective in RLVR~\cite{Zhu2025TheSE} across Figures~\hyperref[fig:decompose]{2}.
The key difference is that NHR converges about $1.5\times$ faster than NLR and attains stronger Avg@$32$ performance in Figure~\hyperref[fig:decompose]{2d}.
Empirically, the benefit of negative feedback is concentrated in NHR rather than NLR, pointing to suppressive or redirecting effects at failed high-entropy branches rather than uniform penalties on routine low-entropy tokens.

\paragraph{Hindsight-Capacity Readout.}
Together, these interventions support the signed-capacity predictions from Section~\ref{sec:grpo_bias}: low-entropy quadrants (PLR/NLR) behave as low-capacity regions with limited or short-lived gains, while signed high-entropy quadrants (PHR/NHR) show stronger sustained improvements after reward polarity resolves the update direction.
Appendix~\ref{app:decompose_exp} gives additional training and test-time evidence.
We next turn to the gradient level to connect this pattern to GRPO's update rule.

\subsection{Gradient-Level Analysis of Token Updates}
\label{sec:gradient_analysis}

For notational simplicity, we fix a trajectory and drop the trajectory index $i$.
Consider the simplified GRPO loss for a sampled token $o_t$ at position $t$.
Let $z_v$ denote the logit of token $v \in \mathcal{V}$ and $\pi_{o_t} = \pi_\theta(o_t | q, o_{<t})$.
The standard softmax derivation yields the token-level gradient descent (full derivation in Appendix~\ref{app:gradient_analysis}):
\begin{equation}
  -\nabla_{z_v} \mathcal{L}_t \approx \begin{cases}
  \hat{A}_t \cdot (1 - \pi_{o_t}), & \text{if} \ v = o_t \\
  - \hat{A}_t \cdot \pi_v, & \text{if} \ v \neq o_t
\end{cases}
\label{eq:main_gradient}
\end{equation}
The critical factor is $(1 - \pi_{o_t})$. In expectation, $\mathbb{E}_{o_t \sim \pi}[1-\pi_{o_t}] = 1-\sum_v \pi_v^2$, so less concentrated high-entropy regimes tend to be more gradient-actionable on average. The CMI bound adds a complementary filter: only positions with sufficient entropy can support a sizable reward-conditioned posterior-prior shift. Thus, high-entropy updates are plausible candidate locations for hindsight credit, while low-entropy updates have limited hindsight capacity. By contrast, GRPO still assigns the same advantage to both. This helps explain the stronger PHR/NHR readouts and the more limited role of PLR/NLR in Section~\ref{sec:decompose_token_effect}, motivating a practical objective that preserves reward polarity while applying bounded, capacity-guided modulation.

\section{Hindsight-Aware Policy Optimization}
\label{method}

The preceding analysis motivates a bounded proxy allocation rule for GRPO.
The ideal hindsight correction depends on the reward-conditioned posterior-prior shift, but this shift is unavailable during training.
\modelname does not estimate this posterior.
Instead, it uses two observable quantities already present in GRPO: reward polarity as the correction sign and token entropy as the capacity proxy.
Viewed from the Four Quadrant Decomposition, \textbf{Hindsight-Aware Policy Optimization (\modelname)} is a soft relaxation of the diagnostic intervention.
Rather than hard-masking one quadrant, it continuously reallocates update magnitude toward signed high-entropy positions while preserving GRPO's optimization structure.


\paragraph{Entropy-Normalized Capacity Proxy.}
Proposition~\ref{prop:credit_bound} gives an upper bound on reward-relevant token information, not the exact hindsight credit.
Accordingly, \modelname uses token entropy as a stable allocation prior for how much of the trajectory-level advantage should be exposed at each position.
Let $\mu_{\mathcal{H}}$ and $\sigma_{\mathcal{H}}$ denote the mean and standard deviation of the token entropies $\{\mathcal{H}_{i,t}\}$ within the sampled group.
We define
\begin{equation}
  C_{i,t} = \text{clip}\left(\frac{\mathcal{H}_{i,t} - \mu_{\mathcal{H}}}{\sigma_{\mathcal{H}}}, -\frac{\left|\hat{A}_{i,t}\right|}{\varphi } ,\frac{\left|\hat{A}_{i,t}\right|}{\varphi }\right),
\end{equation}
where $\varphi>1$ bounds the capacity proxy, and $C_{i,t}$ is positive for above-average-entropy tokens and negative for below-average ones.

\paragraph{Sign-Preserving Advantage Shaping.}
We then incorporate this proxy into GRPO through a shaped advantage:
\begin{equation}
  \hat{A}_{i,t}^{\mathrm{h}} = \hat{A}_{i,t} + \alpha \cdot \operatorname{sign}(\hat{A}_{i,t}) \cdot \operatorname{detach}(C_{i,t}),
\end{equation}
where $\alpha$ controls the modulation strength and $\operatorname{sign}(\cdot)$ follows the reward polarity.
Equivalently, $\hat{A}_{i,t}^{\mathrm{h}}=\operatorname{sign}(\hat{A}_{i,t})(|\hat{A}_{i,t}|+\alpha\operatorname{detach}(C_{i,t}))$.
Given $\alpha \in (0,1]$ and $\varphi>1$, clipping $C_{i,t}$ by $|\hat{A}_{i,t}|/\varphi$ lets the shaped advantage preserves the sign of $\hat{A}_{i,t}$.

The resulting effect mirrors the diagnostic pattern: signed high-entropy updates are strengthened, while routine low-entropy updates are dampened.
The $\operatorname{detach}(\cdot)$ operator (i.e., stop-gradient) prevents entropy from becoming a separate optimization target. It only rescales the policy gradient.
In practice, we replace $\hat{A}_{i,t}$ in Equation~\ref{eq:grpo} with $\hat{A}_{i,t}^{\mathrm{h}}$ and keep the rest of GRPO unchanged.
Hyperparameter sensitivity, including the effect of $\alpha$, is deferred to Appendix~\ref{app:hyperparameter_full}.

\begin{table*}
  [th]
  \setlength{\tabcolsep}{2pt}
  \resizebox{1\linewidth}{!}{
  \begin{tabular}{lccccccccccccc}
    \toprule \noalign{\vskip -2pt} \multirow{2}{*}{Method}                                                                                                    & \multicolumn{2}{c}{AIME24} & \multicolumn{2}{c}{AIME25} & \multicolumn{2}{c}{AMC} & \multicolumn{2}{c}{MATH} & \multicolumn{2}{c}{Minerva} & \multicolumn{2}{c}{Olympiad} & \multicolumn{1}{c}{}    \\
    \cmidrule(lr){2-3}\cmidrule(lr){4-5}\cmidrule(lr){6-7}\cmidrule(lr){8-9}\cmidrule(lr){10-11}\cmidrule(lr){12-13}\cmidrule(lr){14-14} \noalign{\vskip -4pt} 
                                                                                                                                                              & \small{$\text{Avg@}32$}    & \small{$\text{Pass@}32$}   & \small{$\text{Avg@}32$} & \small{$\text{Pass@}32$} & \small{$\text{Avg@}32$}     & \small{$\text{Pass@}32$}     & \small{$\text{Avg@}4$} & \small{$\text{Pass@}4$} & \small{$\text{Avg@}4$} & \small{$\text{Pass@}4$} & \small{$\text{Avg@}4$} & \small{$\text{Pass@}4$} & \cellcolor{col_color}{\texttt{avg}} \\
    \hline
    \noalign{\vskip 1pt}

\rowcolor{row_color!80}\multicolumn{14}{c}{\textbf{\textit{{Classic RLVR Method}}}}                                                  \\
    \noalign{\vskip 1pt} Base-7B                                                                                                                              & 10.0                       & 39.3                       & 5.1                     & 23.5                     & 32.7                        & 85.2                         & 46.9                   & 66.3                    & 14.3                   & 25.9                    & 15.5                   & 27.1                    & \cellcolor{col_color}{20.8}         \\
    PRIME                                                                                                                                                     & 15.3                       & 54.5                       & 6.1                     & 26.3                     & 42.4                        & 90.5                         & 56.6                   & 74.1                    & 15.7                   & 37.8                    & 20.8                   & 34.7                    & \cellcolor{col_color}30.1           \\
    W-REINF                                                                                                                                                   & 31.2                       & 58.2                       & 10.1                    & 34.0                     & 58.1                        & {87.1}                       & 76.2                   & 84.5                    & 34.2                   & 51.2                    & 38.9                   & 49.3                    & \cellcolor{col_color}41.6           \\
    GRPO                                                                                                                                                      & 34.0                       & 50.8                       & 9.3                     & 25.2                     & 58.6                        & 80.3                         & 79.9                   & 85.1                    & 38.1                   & 43.4                    & 42.8                   & 51.1                    & \cellcolor{col_color}43.8           \\
    DAPO                                                                                                                                                      & 35.7                       & \underline{67.4}           & \underline{17.1}        & \textbf{40.8}            & 61.0                        & 84.8                         & 82.6                   & 88.1                    & 42.8                   & 50.4                    & 43.9                   & 52.8                    & \cellcolor{col_color}47.2           \\
    \hline
    \noalign{\vskip 1pt} \rowcolor{row_color!80}\multicolumn{14}{c}{\textbf{\textit{Policy Entropy Method}}}                                                   \\
    \noalign{\vskip 1pt} EntroReg                                                                                                                             & 34.6                       & 52.9                       & 12.6                    & 28.5                     & 61.1                        & 84.9                         & 82.7                   & 87.7                    & 43.1                   & 50.8                    & 42.8                   & 51.4                    & \cellcolor{col_color}46.2           \\
    Forking                                                                                                                                                   & 35.8                       & 58.1                       & 15.1                    & 36.1                     & \underline{63.7}            & 86.6                         & \textbf{84.1}          & \textbf{89.4}           & \underline{43.0}       & 50.1                    & \textbf{45.5}          & \textbf{54.4}           & \cellcolor{col_color}47.9           \\
    EntroAdv                                                                                                                                                  & 36.8                       & 65.0                       & 16.3                    & 35.7                     & \textbf{63.8}               & \textbf{89.4}                & 83.5                   & 88.6                    & 42.7                   & 49.3                    & 44.5                   & 53.6                    & \cellcolor{col_color}47.9           \\
    KL-Cov                                                                                                                                                    & \underline{38.9}           & 61.6                       & 13.8                    & 33.6                     & 59.2                        & 85.9                         & 81.5                   & 87.2                    & {40.9}                 & \underline{54.1}        & 44.2                   & 51.8                    & \cellcolor{col_color}46.4           \\
    \rowcolor{gray!10}\modelname                                                                                                                              & \textbf{39.8}              & \textbf{68.9}              & \textbf{17.2}           & \underline{37.6}         & 62.1                        & \underline{88.8}             & \underline{83.7}       & \underline{88.7}        & \textbf{43.9}          & \textbf{57.3}           & \underline{45.1}       & \underline{53.8}        & \cellcolor{col_color}\textbf{48.6}  \\
    \bottomrule
  \end{tabular}
  }
  \caption{Main results in percentage (\%) on mathematics reasoning benchmarks
  with \texttt{Qwen2.5-Math-7B}. The best results are in \textbf{bold}, and the
  second-best results are \underline{underlined}. All subsequent tables
  follow this format. The \texttt{avg} is
  computed by averaging the $\text{Avg@}k$ across all benchmarks. }
  \label{tab:qwen_main_results}
\end{table*}

\begin{table*}
  [th]
  \setlength{\tabcolsep}{2pt}
  \resizebox{1\linewidth}{!}{
  \begin{tabular}{lccccccccccccc}
    \toprule \noalign{\vskip -2pt} \multirow{2}{*}{Method}                                                                                 & \multicolumn{2}{c}{AIME24} & \multicolumn{2}{c}{AIME25} & \multicolumn{2}{c}{AMC} & \multicolumn{2}{c}{MATH} & \multicolumn{2}{c}{Minerva} & \multicolumn{2}{c}{Olympiad} & \multirow{2}{*}{}       \\
    \cmidrule(lr){2-3}\cmidrule(lr){4-5}\cmidrule(lr){6-7}\cmidrule(lr){8-9}\cmidrule(lr){10-11}\cmidrule(lr){12-13} \noalign{\vskip -4pt} & \small{$\text{Avg@}32$}    & \small{$\text{Pass@}32$}   & \small{$\text{Avg@}32$} & \small{$\text{Pass@}32$} & \small{$\text{Avg@}32$}     & \small{$\text{Pass@}32$}     & \small{$\text{Avg@}4$} & \small{$\text{Pass@}4$} & \small{$\text{Avg@}4$} & \small{$\text{Pass@}4$} & \small{$\text{Avg@}4$} & \small{$\text{Pass@}4$} & \cellcolor{col_color}{\texttt{avg}} \\
    \hline
    \noalign{\vskip 1pt} \rowcolor{row_color!80}\multicolumn{14}{c}{\textbf{\textit{{Classic RLVR Method}}}}                                \\
    \noalign{\vskip 1pt} Base-1.5B                                                                                                         & 32.5                       & 71.6                       & 24.3                    & 51.9                     & 69.4                        & 89.9                         & 85.7                   & 91.9                    & 37.5                   & 46.7                    & 55.2                   & 63.8                    & \cellcolor{col_color}50.7           \\
    W-REINF                                                                                                                                & 35.0                       & 66.1                       & 25.0                    & 49.3                     & 69.8                        & 91.0                         & 87.5                   & 92.6                    & 37.8                   & 37.8                    & 55.9                   & 65.2                    & \cellcolor{col_color}51.8           \\
    GRPO                                                                                                                                   & 38.1                       & 64.8                       & 27.1                    & 49.1                     & 71.4                        & 92.4                         & 88.9                   & 93.4                    & 39.4                   & {48.0}                  & 58.7                   & 67.4                    & \cellcolor{col_color}53.9           \\
    DAPO                                                                                                                                   & 40.9                       & \textbf{75.6}              & 28.6                    & 53.6                     & {73.6}                      & 90.8                         & 89.9                   & 94.0                    & 39.1                   & {48.0}                  & 58.4                   & 67.1                    & \cellcolor{col_color}55.1           \\
    \hline
    \noalign{\vskip 1pt} \rowcolor{row_color!80}\multicolumn{14}{c}{\textbf{\textit{Policy Entropy Method}}}                                \\
    \noalign{\vskip 1pt} EntroReg                                                                                                          & 34.6                       & 70.4                       & 26.4                    & 49.2                     & 70.2                        & 90.6                         & 86.8                   & {92.5}                  & 37.6                   & 47.5                    & 56.2                   & 67.0                    & \cellcolor{col_color}51.9           \\
    Forking                                                                                                                                & 38.1                       & {71.7}                     & 28.1                    & \textbf{55.5}            & 71.0                        & 90.7                         & 87.6                   & 93.1                    & 38.7                   & 48.6                    & 59.1                   & 68.4                    & \cellcolor{col_color}53.7           \\
    EntroAdv                                                                                                                               & \underline{44.0}           & 73.9                       & \textbf{30.4}           & {54.5}                   & 73.9                        & \underline{92.9}             & \underline{90.2}       & \underline{94.5}        & \underline{40.4}       & {46.2}                  & {59.1}                 & {68.9}                  & \cellcolor{col_color}56.3           \\
    KL-Cov                                                                                                                                 & 43.9                       & 71.8                       & \underline{30.1}        & {54.9}                   & \underline{75.0}            & {91.0}                       & 90.0                   & 94.2                    & \textbf{40.5}          & \textbf{50.2}           & \underline{59.9}       & \underline{69.0}        & \cellcolor{col_color}56.5           \\
    \rowcolor{gray!10}\modelname                                                                                                           & \textbf{45.1}              & \underline{74.6}           & \underline{30.1}        & \underline{55.1}         & \textbf{75.5}               & \textbf{95.5}                & \textbf{91.1}          & \textbf{95.2}           & {39.7}                 & \underline{49.1}        & \textbf{60.8}          & \textbf{69.5}           & \cellcolor{col_color}\textbf{57.0}  \\
    \bottomrule
  \end{tabular}
  }
  \caption{Main results in percentage (\%) on mathematics reasoning benchmarks
  with \texttt{DeepScaleR-1.5B-8k}. }
  \label{tab:ds_main_results}
\end{table*}

\section{Experiments and Analysis}
\label{sec:exp}

\subsection{Experiment Setup}

\paragraph{Training.} 
\modelname is evaluated under two LLM settings: \texttt{Qwen2.5-Math-7B} for standard-length 4k math reasoning and \texttt{DeepScaleR-1.5B-8k} for long-chain reasoning up to 24k tokens~\cite{DBLP:journals/corr/abs-2409-12122,Guo2025Deepseek}.
Following prior work~\cite{DBLP:journals/corr/abs-2506-02208,DeepScaleR}, both models are trained on the DeepScaleR dataset~\cite{DeepScaleR}.
During rollout, decoding uses a temperature of $1.0$ and top-p of $0.95$.
For \modelname, the scaling coefficient is $\alpha=0.2$ and the clipping range is $\varphi=2$.
Full implementation details appear in Appendix~\ref{app:training_hyperparameters}.

\noindent\textbf{Baselines.}
The comparison includes strong baselines discussed in Related Work Appendix~\ref{app:related_work}: classic RLVR methods
(PRIME~\cite{DBLP:journals/corr/abs-2502-01456}, W-REINF~\cite{Zhu2025TheSE}, GRPO~\cite{DBLP:journals/corr/abs-2402-03300}, and DAPO~\cite{DBLP:journals/corr/abs-2503-14476})
and entropy-aware interventions
(EntroReg~\cite{DBLP:journals/corr/abs-2506-14758}, Forking~\cite{DBLP:journals/corr/abs-2506-01939}, KL-Cov~\cite{DBLP:journals/corr/abs-2505-22617}, and EntroAdv~\cite{DBLP:journals/corr/abs-2506-14758}).
Baseline implementation details appear in Appendix~\ref{app:baseline_implementation}, and Appendix~\ref{app:method_comparison} gives the systematic method comparison.

\noindent\textbf{Evaluation.}
Evaluation covers six widely used reasoning benchmarks:
AIME24, AIME25, AMC, MATH~\cite{DBLP:conf/nips/HendrycksBKABTS21}, Minerva~\cite{DBLP:conf/nips/LewkowyczADDMRS22}, and Olympiad~\cite{DBLP:conf/acl/HeLBHTSHHHZLQL024}.
Inference uses vLLM~\cite{DBLP:conf/sosp/KwonLZ0ZY0ZS23} with a temperature $0.6$ and top-p $0.95$.
To reduce variance in stochastic decoding, we report $\text{Avg@}32$ and $\text{Pass@}32$ for AIME24, AIME25, and AMC, and $\text{Avg@}4$ and $\text{Pass@}4$ for the remaining benchmarks.
Appendix~\ref{app:evaluation_setup} reports additional benchmark statistics and evaluation details.

\subsection{Main Results}

Tables~\ref{tab:qwen_main_results} and~\ref{tab:ds_main_results} detail the reasoning performance of LLMs across various benchmarks. 

\paragraph{Performance gains concentrate on hard tasks.}
Section~\ref{sec:info_theory} frames entropy as a necessary-condition signal for where large hindsight shifts can reside, rather than an estimator of token credit itself.
Harder reasoning tasks contain more uncertain branch points, giving the signed-capacity proxy more opportunities to correct the uniform reward broadcast in GRPO.
Accordingly, \modelname achieves its most significant gains on demanding benchmarks.
On \texttt{Qwen2.5-Math-7B}, it reaches 39.8 Avg@$32$ on AIME24 (vs.\ DAPO 35.7, EntroAdv 36.8) and posts the highest aggregate score of 48.6\%.
This trend persists on \texttt{DeepScaleR-1.5B-8k} with longer reasoning traces (up to 24k tokens), where \modelname secures a 57.0\% aggregate despite increased optimization difficulty~\cite{DBLP:journals/corr/abs-2502-03373}.
It also maintains competitive performance on partially saturated benchmarks (AMC, MATH), where fewer uncertain branch points make credit assignment a less dominant bottleneck.

\paragraph{Accuracy and diversity improve jointly.}
Our four-quadrant analysis (Section~\ref{sec:decompose_token_effect}) reveals that PHR and NHR interventions improve different parts of the Avg/Pass profile without directly observing latent token credit.
Tables~\ref{tab:qwen_main_results} and~\ref{tab:ds_main_results} show that HAPO improves the Avg/Pass profile relative to GRPO across the two model settings.
In contrast, baselines like DAPO often sacrifice $\text{Avg@}k$ to boost $\text{Pass@}k$, showing how less targeted update rules can trade accuracy for diversity rather than enhancing both.

\paragraph{Comparison with entropy-aware alternatives.}
Forking uses a binary high-entropy mask, and EntroAdv adds an entropy bonus.
Both surpass GRPO, but their single-dimensional interventions remain suboptimal.
Visualizations in Appendix~\ref{app:credit_assignment_analysis} illustrate that \modelname applies continuous four-quadrant modulation, using bounded proxy allocation instead of relying on hard cutoffs.
Appendix~\ref{app:entropy_dynamics} further shows that \modelname maintains stable entropy throughout training, avoiding the entropy collapse typical of GRPO.

Out-of-domain evaluation on MMLU-Pro mirrors these main results. As detailed in Appendix~\ref{app:mmlu_pro}, \modelname achieves the best overall accuracy (36.6\%), indicating that the bounded proxy allocation may transfer beyond mathematical reasoning.

\subsection{Efficiency Analysis}
Efficiency metrics offer complementary evidence. By concentrating update budget where entropy leaves room for signed hindsight correction, \modelname can improve training and inference efficiency.

\begin{figure}[h]
  \centering
  \includegraphics[width=\linewidth]{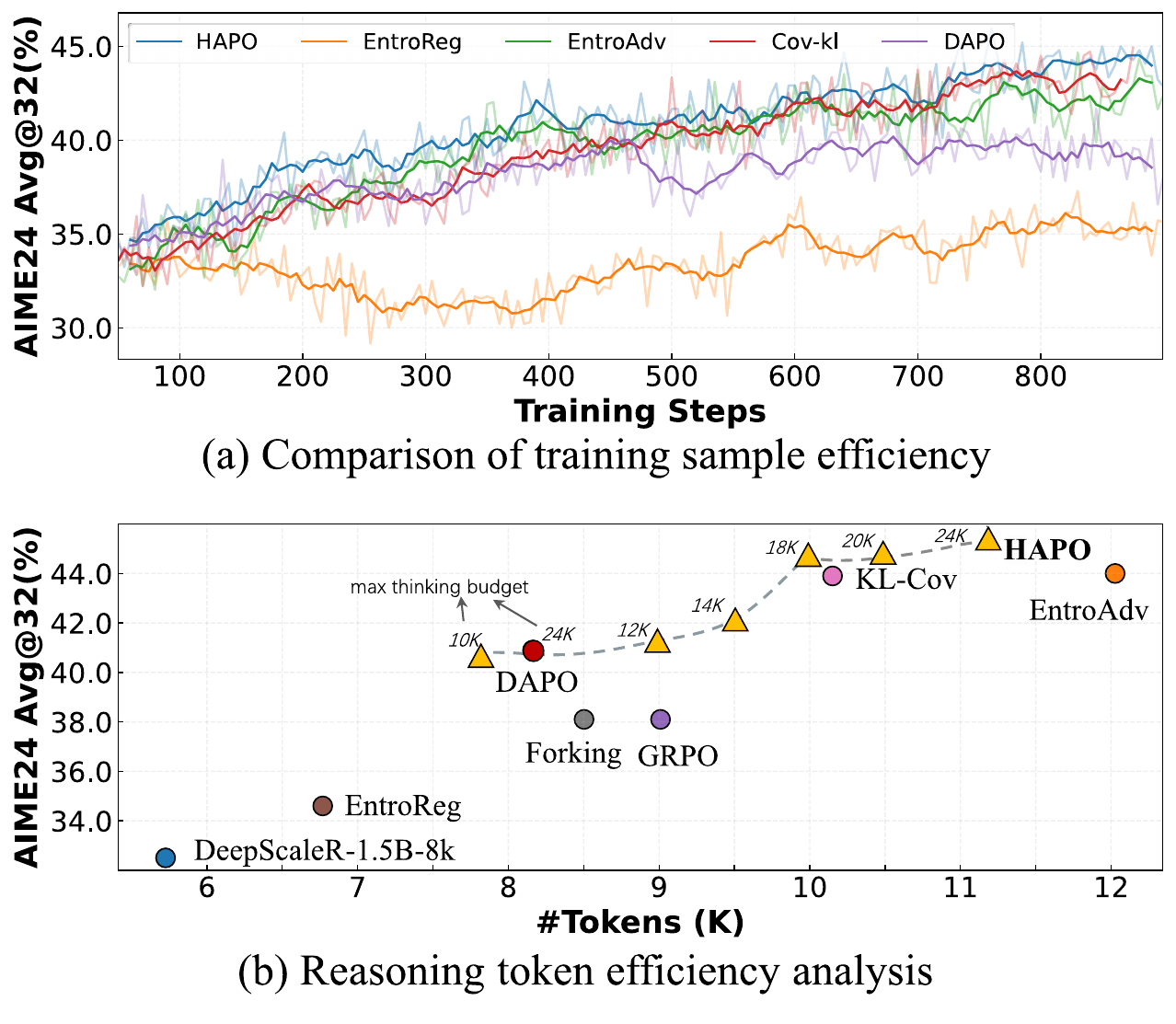}
  \vspace{-0.8cm}
  \caption{Performance--efficiency trade-off for 1.5B models on AIME24. We plot \modelname 's Pareto frontier under different max response length budget.}
  \label{fig:effi_compare}
\end{figure}

Figure~\ref{fig:effi_compare} exhibits the same trend on AIME24. During training, \modelname outpaces GRPO and DAPO, highlighting good sample efficiency. At test time, it avoids EntroAdv's over-generation, yielding a better Pareto frontier between accuracy and token usage. Taken together, these improvements are better explained by bounded proxy allocation than by longer reasoning chains alone. Additional efficiency evidence on MATH and Olympiad appears in Appendix~\ref{app:efficiency_additional}.

\subsection{Ablation Study}
This ablation masks the \modelname proxy term for one quadrant (PHR, PLR, NHR, or NLR), reverting those tokens to vanilla GRPO updates, and asks whether the same high-entropy proxy pattern persists.
The experiments are averaged over three trials.
\begin{table}[h]
\setlength\tabcolsep{2pt}
\resizebox{1\linewidth}{!}{
\begin{tabular}{lcccccccccccc}
\toprule
{}     & {AIME24/25} & {AMC} & {MATH} & {Minerva} & {Olympiad} & {\texttt{avg}}  \\ \hline
\modelname               & \textbf{39.8}/\textbf{17.2}                 & 62.1        & {83.7}               & \textbf{43.9}                  & \textbf{45.1}                & \textbf{48.6}                            \\
        \  \textbf{\textit{w/o} PHR} & 35.6/15.6  & \textbf{63.4}  & \textbf{83.9}  & 42.9  & 43.2  & 47.4  \\ 
        \  \textbf{\textit{w/o} PLR} & 39.2/15.8  & 63.1  & 83.2  & 43.1  & 44.5  & 48.1  \\ 
        \  \textbf{\textit{w/o} NLR} & 36.7/17.1  & 61.2  & 83.2  & \textbf{43.9}  & 43.8  & 47.7  \\ 
        \  \textbf{\textit{w/o} NHR} & 35.7/16.3  & 61.8  & 82.8  & 44.0 & 42.9  & 47.2 \\ 
\bottomrule
\end{tabular}
}
\caption{Ablation study of \modelname with 7B model. Avg@$32$ for AIME24, AIME25, AMC, and Avg@$4$ for other benchmarks are reported in percentage (\%).}
\label{tab:ablation}
\end{table}

\begin{itemize}[leftmargin=*, nosep]
    \item \textbf{High-Entropy (PHR \& NHR).} \textit{w/o} NHR yields the lowest aggregate score (47.2\%), and \textit{w/o} PHR reduces the score to 47.4\%. These drops indicate that signed high-entropy modulation contributes the most useful proxy signal in both failed and successful trajectories.
    \item \textbf{Low-Entropy Quadrants (PLR \& NLR).} Removing low-entropy components produces only modest changes (e.g., \textit{w/o} PLR causes a $\sim$0.5\% drop). This pattern matches the entropy ceiling: low-entropy positions cannot support large $I(\textbf{o}_{i,t}; r | s_{i,t})$, so removing their modulation changes the objective less.
\end{itemize}
These ablations mirror the diagnostic pattern in Section~\ref{sec:decompose_token_effect} and indicate that the useful proxy signal is predominantly concentrated in high-entropy regions.
Appendix~\ref{app:hyperparameter_full} reports hyperparameter sensitivity.

\section{Conclusion}
This paper studies token-level credit assignment in RLVR from a hindsight perspective.
This CMI adaptation shows that the expected posterior-prior shift induced by the final reward is upper-bounded by token entropy, making entropy a tractable ceiling rather than an exact credit estimator.
The Four Quadrant Decomposition and gradient analysis provide diagnostic evidence that sustained gains are more pronounced in signed high-entropy updates, motivating \modelname as a simple, sign-preserving modification to GRPO.
Across mathematical reasoning benchmarks and two model settings, \modelname improves accuracy, diversity, and efficiency while preserving the simplicity of GRPO.
The polarity-entropy framework can also serve as a diagnostic tool for credit assignment in RLVR beyond the specific instantiation presented here.

\section*{Limitations}
We identify several main limitations of our work.
First, entropy is a one-sided capacity signal rather than a direct estimator of token-level hindsight credit.
\modelname therefore does not estimate the true posterior-prior ratio; it implements bounded, sign-preserving proxy allocation.
Similarly, the quadrant decomposition is diagnostic and should be interpreted as intervention evidence rather than direct observation of latent token credit.
Second, our method relies on outcome-based verifiable rewards, which are readily available in domains like mathematics. Preliminary results on MMLU-Pro (Appendix~\ref{app:mmlu_pro}) suggest broader applicability.
Third, \modelname introduces additional hyperparameters, $\alpha$ and $\varphi$. Appendix~\ref{app:hyperparameter_full} reports sensitivity analysis for $\alpha$, and our experiments fix $\varphi=2$.

\section*{Ethical Considerations}
This work focuses on enhancing the mathematical reasoning capabilities of Large Language Models through improved reinforcement learning techniques.
While our primary goal is to advance the state of automated reasoning, we acknowledge potential ethical implications, such as the consumption of computational resources and energy. We rely exclusively on publicly available datasets and benchmarks, which contain no private or sensitive information. We confirm that our use of these artifacts aligns with their intended use and access conditions. We also strictly adhere to the licensing terms of all utilized models and data.
Our use of existing artifacts, including \texttt{Qwen2.5-Math-7B}, \texttt{DeepScaleR-1.5B-8k}, the mathematical reasoning datasets (AIME24/25, AMC, MATH, Minerva, Olympiad), and MMLU-Pro, falls within their intended use for research and development purposes.
To promote transparency and reproducibility, we intend to release our code and derivative artifacts under a permissive open-source license.

\bibliography{custom}

\appendix

\section{Related Works}
\label{app:related_work}

\subsection{Reinforcement Learning with Verifiable Rewards}
Reinforcement Learning with Verifiable Rewards (RLVR) has emerged as a scalable and effective paradigm for enhancing the reasoning capabilities of Large Language Models (LLMs)~\cite{DBLP:journals/corr/abs-2504-05185,DBLP:journals/corr/abs-2503-18892,DBLP:journals/corr/abs-2503-06639}.
Unlike Reinforcement Learning from Human Feedback (RLHF)~\cite{DBLP:conf/nips/ChristianoLBMLA17,DBLP:journals/ail/Lee25}, which relies on preference models, RLVR utilizes rule-based outcome rewards, making it particularly suitable for domains with objective correctness such as mathematics and coding.
Recent successes, including DeepSeek-R1~\cite{Guo2025Deepseek} and Kimi 1.5~\cite{DBLP:journals/corr/abs-2501-12599}, demonstrate that RLVR can incentivize long-chain reasoning and improve generalization beyond supervised fine-tuning.
However, most RLVR approaches, such as GRPO~\cite{DBLP:journals/corr/abs-2402-03300} and its variants~\cite{DBLP:journals/corr/abs-2506-01939,DBLP:journals/corr/abs-2504-02546,DBLP:journals/corr/abs-2505-22617}, typically assign sequence-level rewards uniformly across all tokens.
This uniform assignment overlooks the distinct roles tokens play in reasoning chains, exacerbating the credit assignment problem~\cite{Minsky1995StepsTA}.
Recent studies have analyzed token patterns in reasoning~\cite{DBLP:journals/corr/abs-2506-01939,DBLP:journals/corr/abs-2507-15778}, but the mechanism by which outcome rewards should be allocated across token positions remains underexplored.
In this work, we focus on outcome-based RLVR and investigate how sequence-level rewards manifest as token-level update effects.

\subsection{Credit Assignment in Outcome-Based RLVR}
The credit assignment problem—identifying which specific steps contribute to a final success or failure—remains a fundamental bottleneck in training long-horizon reasoning LLMs~\cite{DBLP:journals/corr/abs-2211-14275,DBLP:conf/iclr/LightmanKBEBLLS24}.
Traditional approaches often employ process-based RL with step- or token-level supervision, such as Process Reward Models (PRMs)~\cite{DBLP:journals/corr/abs-2502-01456,DBLP:conf/acl/WangLSXDLCWS24} and value estimation~\cite{DBLP:conf/icml/KazemnejadAPSRC25,tran2025exploiting}.
However, these methods suffer from high annotation and training costs, as well as susceptibility to reward hacking~\cite{DBLP:conf/acl/WangLSXDLCWS24, DBLP:conf/iclr/LightmanKBEBLLS24}.
Consequently, recent research has shifted towards refining credit assignment within outcome-based RLVR.
One line of work dissects the disparate impacts of positive and negative samples~\cite{chen2025passktrainingadaptivelybalancing,Zhu2025TheSE,DBLP:conf/acl/ZhangZWZLYLZL25}, with some arguing that reinforcing negative samples is surprisingly more effective for generalization~\cite{Zhu2025TheSE}.
Another direction differentiates token update magnitude based on entropy~\cite{DBLP:journals/corr/abs-2506-01939,DBLP:journals/corr/abs-2505-15134,DBLP:journals/corr/abs-2509-09265}.
However, these works yield conflicting views on whether high-entropy or low-entropy tokens should be prioritized for learning~\cite{DBLP:journals/corr/abs-2506-01939,DBLP:journals/corr/abs-2505-15134}.
The Four Quadrant Decomposition helps reconcile these perspectives by separating reward polarity from token entropy and offering a diagnostic framework for studying signed token update effects.

\subsection{Information-Theoretic Perspectives}
Beyond empirical analyses, a growing body of theoretical work seeks principled foundations for credit assignment in RL.
\citet{arumugam2021information} formalized action credit through Conditional Mutual Information (CMI) between actions and returns. It implies that policy entropy upper-bounds the information an action can carry about the outcome.
This result, originally developed for general MDPs, suggests a natural interpretation for RLVR:
many routine tokens in a reasoning chain are near-deterministic, so the CMI bound rules out large expected reward-conditioned shifts at these positions.
This work maps the CMI framework to the RLVR setting (Section~\ref{sec:info_theory}), deriving signed-capacity predictions that our four-quadrant experiments subsequently probe.
Concurrently, \citet{TokenHiddenReward2025} introduced the Token Hidden Reward framework, which learns an implicit per-token reward signal and found that reward-carrying tokens overlap with high-entropy positions by over 90\%---empirical evidence in line with the hindsight-capacity view.
\citet{LearningDynamicsLLM2025} analyzed learning dynamics during fine-tuning and identified a ``squeezing effect'' at low-entropy positions, matching our gradient analysis showing that the factor $(1-\pi_{o_t})$ vanishes as confidence increases.
These complementary perspectives converge on the same conclusion: entropy is a useful capacity signal for locating where non-trivial hindsight shifts may reside, even though it is not itself an exact credit estimator.

\subsection{The Entropy Interventions Approach in RLVR}
Policy entropy has long served as a regularizer to promote exploration in reinforcement learning~\cite{DBLP:phd/us/Ziebart18,DBLP:conf/aaai/ZiebartMBD08}.
In RLVR, entropy is increasingly used to shape token updates, improving training efficiency and stability.
Prior work largely falls into two camps:
(1) \textbf{Exploration-centric methods}, which amplify gradients at high-entropy branching points (e.g., Forking~\cite{DBLP:journals/corr/abs-2506-01939}, EntroAdv~\cite{DBLP:journals/corr/abs-2506-14758}); and
(2) \textbf{Stability-centric methods}, which prioritize low-entropy tokens to ground knowledge and prevent model collapse~\cite{DBLP:journals/corr/abs-2509-09265,DBLP:journals/corr/abs-2505-15134,DBLP:journals/corr/abs-2505-12929}.
Our work bridges these entropy interventions.
\citet{DBLP:journals/corr/abs-2506-01939} reported the ``80/20 rule'' (a minority of tokens drive reasoning), and our Four Quadrant Decomposition offers one capacity-based interpretation of why this is plausible: the CMI bound (Proposition~\ref{prop:credit_bound}) rules out large expected hindsight shifts at low-entropy positions, so the ``important 20\%'' are candidates whose entropy permits non-trivial reward-conditioned shifts.
\modelname does not reject this minority-token phenomenon; instead, it refines the hard cutoff into continuous, bounded proxy allocation, strengthening signed updates at high-entropy forks while conservatively down-weighting low-entropy segments.
This yields a capacity-guided proxy allocation mechanism that remains compatible with the standard GRPO objective.

\subsection{Detailed Method Comparison}
\label{app:method_comparison}

Table~\ref{tab:method_comparison} summarizes \modelname alongside concurrent and related entropy-based update allocation methods.

\begin{table*}[th]
\centering
\resizebox{0.95\linewidth}{!}{
\begin{tabular}{lccccc}
\toprule
\textbf{Method} & \textbf{Token Criterion} & \textbf{Polarity-Aware} & \textbf{CMI-Bound Grounding} & \textbf{Modulation Type} & \textbf{Auxiliary Model} \\
\midrule
GRPO & None (uniform) & \ding{55} & \ding{55} & Uniform advantage & \ding{55} \\
Forking & Entropy (top-20\%) & \ding{55} & \ding{55} & Binary mask & \ding{55} \\
EntroAdv & Entropy (additive) & \ding{55} & \ding{55} & Additive entropy bonus & \ding{55} \\
KL-Cov & KL divergence & \ding{55} & \ding{55} & Continuous & \ding{55} \\
PRIME & Learned reward & \ding{55} & \ding{55} & Token-level reward & \ding{51} \\
A3PO & Probability (high/low) & \ding{51} & \ding{55} & Adaptive time-decaying & \ding{55} \\
\midrule
\modelname & Group-normalized entropy & \ding{51} & CMI bound & Bounded sign-preserving & \ding{55} \\
\bottomrule
\end{tabular}
}
\caption{Comparison of token-level update allocation methods. ``Polarity-Aware'' indicates whether the method applies different modulations to positive vs.\ negative samples. ``CMI-Bound Grounding'' denotes whether the design is motivated by the entropy capacity bound used in this work.}
\label{tab:method_comparison}
\end{table*}

Among these methods, A3PO~\cite{A3PO2025} is the most related concurrent work, as it is also polarity-aware.
However, the two approaches differ in three key respects.
First, A3PO differentiates tokens by their absolute probability (high- vs.\ low-probability), but \modelname uses group-normalized entropy, which captures relative distributional uncertainty within the sampled group rather than point-wise confidence.
Second, A3PO applies adaptive time-decaying modulation that gradually reduces the effect of its token weighting during training. In contrast, \modelname maintains bounded sign-preserving modulation throughout, preserving the entropy-aware signal at all stages.
Third, \modelname is motivated by the CMI upper bound (Proposition~\ref{prop:credit_bound}), which supports using entropy as a capacity proxy for reward-conditioned correction.

\section{Training Setup}
\label{app:training_hyperparameters}

\subsection{Four Quadrant Training Objective}
\label{app:training_objective}
The Four Quadrant Decomposition consists of four components: \textit{Positive High-entropy Reinforcement} (PHR), \textit{Positive Low-entropy Reinforcement} (PLR), \textit{Negative High-entropy Reinforcement} (NHR), and \textit{Negative Low-entropy Reinforcement} (NLR).
This decomposition allows us to diagnostically investigate how signed high- and low-entropy token updates shape model behavior.
Formally, we decompose the GRPO objective into quadrant-restricted subobjectives based on Equation~\ref{eq:grpo} as follows:
\begin{equation*}
\hspace{-0cm}\resizebox{\hsize}{!}{$\displaystyle
\begin{aligned}
\mathcal{L}_{\mathrm{GRPO}}(\theta) = &- \mathbb{E}_{q\sim Q,\{o_i\}_{i=1}^G\sim\pi_{\theta_{\mathrm{old}}}(\cdot|q)} \\
&\Bigg[\!\frac{1}{\sum_{i=1}^G |o_i|}\!\sum_{i=1}^G\!\sum_{t=1}^{|o_i|} 
\sum_{m\in\{\text{I},\text{II},\text{III},\text{IV}\}}\mathbb{I}_{m}\min\!\Big(\rho_{i,t}\!\cdot\!\hat{A}_{i,t},\\
&\!\operatorname{clip}(\rho_{i,t}, \!1-\epsilon_{\text{low}},\!1+\epsilon_{\text{high}})\hat{A}_{i,t}\Big)\Bigg] \\
= &\  \mathcal{L}_{\mathrm{PHR}}(\theta,\mathbb{I}_{\text{I}}) + \mathcal{L}_{\mathrm{PLR}}(\theta,\mathbb{I}_{\text{II}}) + \\
& \ \mathcal{L}_{\mathrm{NHR}}(\theta,\mathbb{I}_{\text{III}}) + \mathcal{L}_{\mathrm{NLR}}(\theta,\mathbb{I}_{\text{IV}})
\end{aligned}$}
\label{eq:grpo_decompose}
\end{equation*}
where each $\mathbb{I}_{m}$ is an indicator function that evaluates to 1 if the token belongs to quadrant $m$ and 0 otherwise. The four masks are defined as follows:
\begin{equation*}
  \begin{aligned}
  \mathbb{I}_{\text{I}} = 
  \begin{cases}
    1, & \text{if} \  \mathcal{H}_{i,t}\ge\mu_{\mathcal{H}}, r_i\ge0 \\
    0, & \text{otherwise}
  \end{cases}
\end{aligned}
\end{equation*}
\begin{equation*}
  \begin{aligned}
  \mathbb{I}_{\text{II}} = 
  \begin{cases}
    1, & \text{if} \  \mathcal{H}_{i,t}<\mu_{\mathcal{H}}, r_i\ge0 \\
    0, & \text{otherwise}
  \end{cases}
\end{aligned}
\end{equation*}
\begin{equation*}
  \begin{aligned}
  \mathbb{I}_{\text{III}} = 
  \begin{cases}
    1, & \text{if} \  \mathcal{H}_{i,t}\ge\mu_{\mathcal{H}}, r_i<0 \\
    0, & \text{otherwise}
  \end{cases}
\end{aligned}
\end{equation*}
\begin{equation*}
  \begin{aligned}
  \mathbb{I}_{\text{IV}} = 
  \begin{cases}
    1, & \text{if} \  \mathcal{H}_{i,t}<\mu_{\mathcal{H}}, r_i<0 \\
    0, & \text{otherwise}
  \end{cases}
\end{aligned}
\end{equation*}
where $\mu_{\mathcal{H}}$ is the average token entropy over all tokens in the sampled group. In our diagnostic experiments, only the selected quadrant contributes gradients; the decomposition is therefore used to probe where useful signed update effects concentrate, rather than to claim direct observation of true token credit.

\subsection{Quadrant Decomposition Experiments}
All experiments use Verl~\cite{DBLP:conf/eurosys/ShengZYWZZPL025} as the RL training framework and run on a single node with 8 NVIDIA A800 GPUs.
Negative samples use $r_i=-1$, and positive samples use $r_i=+1$.
Thus, the diagnostic masks are defined by the binary outcome reward, while \modelname uses the sign of the group-normalized advantage in its sign-preserving shaping.
The training dataset is DeepScaleR, which consists of 40,315 unique math reasoning problem-answer pairs.
Additionally, we adopt dynamic sampling and Clip-Higher from DAPO~\cite{DBLP:journals/corr/abs-2503-14476}, as they are commonly used in modern RLVR for efficiency and stability.
Although the KL constraint is commonly used to regularize policy updates, the core learning signal still comes from the verifiable reward, so these experiments use GRPO without the KL constraint to isolate the effect of reward learning.
For PHR, PLR, NHR, and NLR, configuration details are listed in Table~\ref{tab:training_hyperparameters_decompose}.
Each diagnostic intervention is run twice, and the reported results are averaged over these runs.
\begin{table}[ht]
  \centering
  \begin{tabular}{lc}
    \toprule
    \textbf{Config}       & \texttt{Qwen2.5-Math-7B}         \\
    \hline
    lr                     & 1e-6                  \\
    kl\_coef               & 0.0                   \\
    max\_prompt\_length                  & 1024                  \\
    max\_response\_length                    & 3072                     \\
    train\_batch\_size              & 256             \\
    ppo\_mini\_batch\_size                & 32                      \\
    clip\_ratio\_low $\epsilon_{low}$             & 0.20                    \\
    clip\_ratio\_high $\epsilon_{high}$                  & 0.28                    \\
    temperature                & 1.0                 \\
    rollout.n $G$                 & 8                     \\
    total\_training\_steps        & 120                \\
    \bottomrule
  \end{tabular}
  \caption{Training hyperparameters for the four-quadrant experiments.}
  \label{tab:training_hyperparameters_decompose}
\end{table}

\subsection{Main Experiments}

The main experiments use two open-source LLMs, \texttt{Qwen2.5-Math-7B}\footnote{https://huggingface.co/Qwen/Qwen2.5-Math-7B} and \texttt{DeepScaleR-1.5B-8k}\footnote{https://huggingface.co/cmu-deepscaler-colab/DeepScaleR-8K}.
\texttt{Qwen2.5-Math-7B} is a base model that generates answers directly, without prefacing them with ``<think>''.
For comparison, we also consider \texttt{DeepScaleR-1.5B-8k}, a representative model that outputs long chain-of-thought reasoning wrapped between ``<think>'' and ``</think>''.
It is a length-compressed, fine-tuned variant from the DeepSeek family, offering better training efficiency.
For the experiments of \modelname, configuration details are listed in Table~\ref{tab:training_hyperparameters_main}. 
The training data comes from DeepScaleR.
For \texttt{Qwen2.5-Math-7B}, the original prompt template is used. For \texttt{DeepScaleR-1.5B-8k}, the suffix ``\texttt{$\backslash$nPlease reason step by step, and put your final answer within $\backslash$boxed\{\}.}'' is appended to the original prompt as the template.
\begin{table}[ht]
  \centering
  \resizebox{\linewidth}{!}{
  \begin{tabular}{lcc}
    \toprule
    \textbf{Config}       & \texttt{Qwen2.5-Math-7B}   &  \texttt{DeepScaleR-1.5B-8k}    \\
    \hline
    lr                     & 1e-6    & 1e-6              \\
    kl\_coef               & 0.0     & 0.0              \\
    max\_prompt\_length                  & 1024      & 2048            \\
    max\_response\_length                    & 3072          & 24576           \\
    train\_batch\_size              & 512     & 64        \\
    ppo\_mini\_batch\_size                & 32     & 32                 \\
    clip\_ratio\_low $\epsilon_{low}$             & 0.20       & 0.20             \\
    clip\_ratio\_high $\epsilon_{high}$                  & 0.28       & 0.28             \\
    temperature                & 1.0        & 1.0         \\
    rollout.n $G$                 & 8       & 16              \\
    total\_training\_steps        & 220       & 900         \\
    \bottomrule
  \end{tabular}
  }
  \caption{Training hyperparameters for \modelname.}
  \label{tab:training_hyperparameters_main}
\end{table}

\subsection{Experiment Setup for Other Baselines}
\label{app:baseline_implementation}
For reimplemented entropy baselines, we use the same training infrastructure and DAPO-style stability techniques, including Clip-Higher, dynamic sampling, token-level policy gradient loss, and overlong penalty.
When a baseline releases model weights (i.e., W-REINF, GRPO, and PRIME), we directly re-evaluate the released weights under our evaluation pipeline. For baselines that require local reproduction, we use the corresponding open-source code and match the reported implementation choices when available.
We do not report PRIME on the 1.5B model because it requires training a policy-matched implicit process reward model.

\noindent\textbf{GRPO}~\cite{DBLP:journals/corr/abs-2402-03300} eliminates the memory-intensive critic model by optimizing the policy using group-based relative advantages derived from a group of rollouts.
For the locally reproduced GRPO baseline, we use the group-based relative advantage objective and set the clipping parameter $\epsilon_{\text{high}}$/$\epsilon_{\text{low}}$ for the policy objective to 0.2.

\noindent\textbf{DAPO}~\cite{DBLP:journals/corr/abs-2503-14476} builds upon and refines the GRPO framework. It addresses entropy collapse in long-chain-of-thought reasoning by introducing techniques like Clip-Higher and dynamically filtering out zero-variance sample groups. 

\noindent\textbf{W-REINF}~\cite{Zhu2025TheSE} explicitly up-weights negative samples in the objective function to preserve generation diversity. We set the positive advantage weight to 0.1, effectively prioritizing the gradient signal from incorrect reasoning paths (negative samples) to prevent mode collapse.

\noindent\textbf{PRIME}~\cite{DBLP:journals/corr/abs-2502-01456} enables dense, token-level reward optimization without an explicit reward model by deriving implicit process rewards from the log-likelihood ratio between the policy and a reference model. We directly report the results of the released model weights.

\noindent\textbf{EntroReg}~\cite{DBLP:journals/corr/abs-2506-14758} augments the policy loss with an entropy term $\mathcal{H}(\pi)$ to encourage uniform exploration across the vocabulary, serving as a baseline for exploration effectiveness.

\noindent\textbf{EntroAdv}~\cite{DBLP:journals/corr/abs-2506-14758} modifies the advantages with a clipped, gradient-detached entropy term to incentivize exploration at high-uncertainty positions. The shaped advantage is computed as $A^{\mathrm{shaped}}_t = A_t + \psi(\mathcal{H}_t)$, where $\psi(\mathcal{H}_t)>0$. The authors' recommended hyperparameters are used.

\noindent\textbf{Forking}~\cite{DBLP:journals/corr/abs-2506-01939} improves training efficiency by restricting policy gradient updates to the top 20\% of tokens with the highest entropy ("forking tokens"), which are treated as candidate high-uncertainty positions for useful updates.

\noindent\textbf{KL-Cov}~\cite{DBLP:journals/corr/abs-2505-22617} integrates a coverage-aware penalty based on the Kullback-Leibler divergence, formulated as $\mathcal{L}_{\text{cov}} = \mathbb{E}_{a \sim \pi} [\min(\mathbb{D}_{KL}(\pi_{\theta} \| \pi_{\text{ref}}), \delta)]$. This keeps exploration within a trust region $\delta$ that covers diverse reasoning paths without diverging into incoherent token sequences.

\section{Evaluation Setup}
\label{app:evaluation_setup}

\subsection{Metrics}
Accuracy based on greedy decoding can be unreliable~\cite{DBLP:journals/corr/abs-2504-07086}.
To comprehensively evaluate the exploitation and exploration performance of LLMs on mathematical reasoning tasks, we employ two primary metrics: $\text{Avg@}k$ and $\text{Pass@}k$.
During evaluation, the model generates responses at a temperature of 0.6 and a top-p of 0.95.
In particular, $\text{Pass@}k$ is defined as the fraction of problems for which at least one correct answer is produced in $k$ independent trials. $\text{Avg@}k$ estimates the expected correctness of a sampled response and serves as a robust proxy for single-sample accuracy.
They are calculated as:
  \begin{align}
    \text{Pass@}k &= \mathbb{E}_{q\sim Q}\left[ 1 - \frac{\binom{n-c}{k}}{\binom{n}{k}}\right]\\
    \text{Avg@}k &= \mathbb{E}_{q\sim Q}\left[ \ \frac{\ c\ }{\ n\ }\ \right]
  \end{align}
where $n$ is the total number of sampled responses used for estimation and $c$ is the number of correct rollouts.
The answer verification functions are implemented from Qwen\footnote{https://github.com/QwenLM/Qwen2.5-Math} and Math-Verify\footnote{https://github.com/huggingface/Math-Verify} for accurate evaluation.

\subsection{Information for Test Benchmarks}
\modelname and baselines are evaluated on six widely used mathematical reasoning benchmarks:
AIME24/25, AMC23\footnote{https://huggingface.co/datasets/AI-MO/aimo-validation-amc}, MATH~\cite{DBLP:conf/nips/HendrycksBKABTS21}, Minerva~\cite{DBLP:conf/nips/LewkowyczADDMRS22}, and Olympiad~\cite{DBLP:conf/acl/HeLBHTSHHHZLQL024}.
The statistics of these benchmarks are summarized in Table~\ref{tab:benchmark_statistics}.

\begin{table}[!ht]
    \centering
    \begin{tabular}{lll}
    \toprule
        \textbf{Benchmark} & \textbf{\#Questions} & \textbf{Level} \\ \hline
        AIME24 & 30 & Olympiad \\ 
        AIME25 & 30 & Olympiad \\ 
        AMC & 83 & Intermediate \\ 
        MATH & 500 & Advanced \\ 
        Minerva & 272 & Graduate \\ 
        Olympiad & 675 & Olympiad \\ \bottomrule
    \end{tabular}
    \caption{Dataset statistics.\label{tab:benchmark_statistics}}
    
\end{table}

\section{Theoretical Analysis}
\label{app:theoretical_analysis}

This appendix gives additional theoretical details for the information-theoretic hindsight analysis introduced in Section~\ref{sec:info_theory}.
It first presents the formal CMI derivation and its implications (Appendix~\ref{app:cmi_derivation}), then analyzes how this manifests in gradient dynamics (Appendix~\ref{app:gradient_analysis}), and finally discusses connections to related theoretical frameworks (Appendix~\ref{app:theory_connections}).

\subsection{HCA--CMI Derivation}
\label{app:cmi_derivation}

We formalize the HCA target in RLVR through the lens of CMI, building upon the information-theoretic framework of~\citet{arumugam2021information}.

\paragraph{Notation and Setup.}
We map the RLVR token generation process to the standard RL formalism as follows.
At each decoding step $t$, the ``state'' is the context $s_{i,t} := (q, o_{i,<t})$, the ``action'' is the sampled choice $o_{i,t} \sim \pi_\theta(\cdot | s_{i,t})$, and the ``return'' is the outcome reward $r$.
We use bold $\textbf{o}_{i,t}$ to denote the random next token in comparison with the observed token $o_{i,t}$.
The policy entropy at step $t$, $\mathcal{H}_{i,t} = H(\textbf{o}_{i,t} | s_{i,t}) = - \sum_{v \in \mathcal{V}} \pi_\theta(v | s_{i,t}) \log \pi_\theta(v | s_{i,t})$, quantifies the model's uncertainty over the next token.
For HCA analysis, $\pi_\theta$ denotes the behavior policy at rollout time; in GRPO this corresponds to $\pi_{\theta_{\mathrm{old}}}$.

\paragraph{Proposition 1 (Hindsight-CMI Connection).}
\textit{For any prefix state $s_{i,t}=(q,o_{i,<t})$, let $h_{\pi_\theta}(\cdot\mid s_{i,t},r)$ denote the reward-conditioned hindsight distribution over the next token. Then}
\begin{equation*}
\begin{aligned}
  &I(\textbf{o}_{i,t};\, r | s_{i,t}) \\ &=
  \mathbb{E}_{r\mid s_{i,t}}
  \left[
  \mathrm{KL}\!\left(h_{\pi_\theta}(\cdot | s_{i,t},r)\, \|\, \pi_\theta(\cdot | s_{i,t})\right)
  \right].
\end{aligned}
\end{equation*}
\begin{proof}
Following HCA's future-conditional distribution definition~\cite{Harutyunyan2019Hindsight}, $h_{\pi_\theta}$ is the posterior over the earlier token after observing the future outcome. By Bayes' rule, its posterior-prior ratio satisfies
\begin{equation*}
\begin{aligned}
\rho^{\mathrm{hs}}(v,s_{i,t},r)
&:=
\frac{h_{\pi_\theta}(v\mid s_{i,t},r)}
      {\pi_\theta(v\mid s_{i,t})} \\
&=
\frac{p(r\mid s_{i,t},\textbf{o}_{i,t}=v)}
      {p(r\mid s_{i,t})}.
\end{aligned}
\end{equation*}
The KL form of CMI gives~\cite{arumugam2021information}
\begin{equation*}
\begin{aligned}
&I(\textbf{o}_{i,t}; r | s_{i,t}) \\
&=
\mathbb{E}_{r\mid s_{i,t}}
\left[
\mathrm{KL}\!\left(p(\textbf{o}_{i,t}\mid s_{i,t},r)\,\|\,
 p(\textbf{o}_{i,t}\mid s_{i,t})\right)
\right].
\end{aligned}
\end{equation*}
Under the HCA notation, this future-conditioned posterior is $h_{\pi_\theta}(\cdot\mid s_{i,t},r)$, while the prior over the next token is the behavior policy $\pi_\theta(\cdot\mid s_{i,t})$; this gives the proposition. Expanding the KL term then gives the log-ratio form:
\begin{equation*}
\resizebox{\linewidth}{!}{$\displaystyle
\begin{aligned}
I(\textbf{o}_{i,t}; r | s_{i,t})
&=
\mathbb{E}_{r\mid s_{i,t}}
\mathbb{E}_{v\sim h_{\pi_\theta}(\cdot\mid s_{i,t},r)}
\left[
\log
\frac{h_{\pi_\theta}(v\mid s_{i,t},r)}
     {\pi_\theta(v\mid s_{i,t})}
\right] \\
&=
\mathbb{E}_{r\mid s_{i,t}}
\mathbb{E}_{v\sim h_{\pi_\theta}(\cdot\mid s_{i,t},r)}
\left[
\log \rho^{\mathrm{hs}}(v,s_{i,t},r)
\right].
\end{aligned}
$}
\end{equation*}
Therefore, CMI is the expectation of the log HCA ratio over outcomes and hindsight-conditioned token choices, establishing it as a distribution-level measure of the expected posterior-prior shift.
\end{proof}

\paragraph{Interpretation.}
Proposition~1 shows that CMI is the outcome-weighted KL shift from the behavior policy to the hindsight distribution. For a sampled token $o_{i,t}$ with observed outcome $r$, the corresponding pointwise hindsight evidence is the log posterior-prior ratio
\begin{equation*}
c^{\mathrm{hs}}
=
\log
\frac{h_{\pi_\theta}(o_{i,t}\mid s_{i,t},r)}
{\pi_\theta(o_{i,t} | s_{i,t})}
=
\log
\frac{p(r | s_{i,t},o_{i,t})}
{p(r | s_{i,t})}.
\end{equation*}
This pointwise evidence is useful for interpretation, but Proposition~1 and Proposition~2 bound only its expectation through the CMI, not every realized log-ratio.
GRPO does not estimate this hindsight distribution or its KL shift. Instead, it broadcasts a coarse outcome-conditioned scalar $\hat{A}_{i,t}$ to all tokens in the trajectory.

\paragraph{Proposition 2 (Entropy Capacity Bound).}
\textit{For any prefix state $s_{i,t} = (q, o_{i,<t})$, the information contribution of the next-token choice to the outcome $r$ satisfies:}
\begin{equation*}
  I(\textbf{o}_{i,t};\, r | s_{i,t}) \leq \mathcal{H}_{i,t}.
\end{equation*}
\begin{proof}
Starting from Proposition~1, write $h_r(v)=h_{\pi_\theta}(v\mid s_{i,t},r)$ and $\pi(v)=\pi_\theta(v\mid s_{i,t})$.
Expanding the hindsight-KL form gives
\begin{align*}
  I(\textbf{o}_{i,t};\, r | s_{i,t})
  &=
  \mathbb{E}_{r\mid s_{i,t}}
  \left[
  \mathrm{KL}\!\left(h_r\,\|\,\pi\right)
  \right] \\
  &=
  \mathbb{E}_{r\mid s_{i,t}}
  \sum_{v\in\mathcal{V}}
  h_r(v)
  \log
  \frac{h_r(v)}
  {\pi(v)}.
\end{align*}
Let $m(v)=\sum_r p(r\mid s_{i,t})h_r(v)$. Separating the two log terms gives
\begin{align*}
  I(\textbf{o}_{i,t};\, r | s_{i,t})
  &=
  -H(\textbf{o}_{i,t}\mid r,s_{i,t}) \\
  &\quad
  -
  \sum_{v\in\mathcal{V}}
  m(v)\log \pi(v).
\end{align*}
By marginalization,
\begin{align*}
  m(v)
  &=
  p(\textbf{o}_{i,t}=v\mid s_{i,t}) \\
  &=
  \pi_\theta(v\mid s_{i,t}).
\end{align*}
Therefore,
\begin{align*}
  I(\textbf{o}_{i,t};\, r | s_{i,t})
  &=
  H(\textbf{o}_{i,t}\mid s_{i,t})
  -
  H(\textbf{o}_{i,t}\mid r,s_{i,t}) \\
  &=
  \mathcal{H}_{i,t}
  -
  \underbrace{H(\textbf{o}_{i,t}\mid r,s_{i,t})}_{\geq\,0}
  \leq
  \mathcal{H}_{i,t}.
\end{align*}
The inequality follows from the non-negativity of conditional entropy~\cite{CoverThomas2006}. Equality holds if and only if the random next token $\textbf{o}_{i,t}$ is a deterministic function of $(r, s_{i,t})$---that is, knowing both the context and the final outcome fully determines the token choice.
\end{proof}

\paragraph{Interpretation for GRPO.}
Proposition~2 establishes that policy entropy $\mathcal{H}_{i,t}$ upper-bounds the CMI, the expected reward-conditioned distributional shift at a token position. Low-entropy tokens therefore leave little room for such a shift, whereas high-entropy tokens provide a larger capacity ceiling. This bound does not imply that high-entropy tokens necessarily carry credit; it only marks where a meaningful hindsight correction is possible. \modelname uses entropy as this training-free capacity proxy.

\paragraph{Diagnostic Readout: Entropy--Reward Association.}
\label{par:proxy_cmi}
Exact computation of $I(\textbf{o}_{i,t}; r | s_{i,t})$ requires the hindsight distribution $h_{\pi_\theta}$, but a diagnostic statistic can be computed from the training data.
Specifically, we measure the mutual information between the token entropy value $\mathcal{H}_{i,t}$ and the binary reward $r_i$:
\begin{equation*}
  \widehat{I}(\mathcal{H}_{i,t};\, r_i) = H(\mathcal{H}_{i,t}) - H(\mathcal{H}_{i,t} | r_i).
\end{equation*}
Note that $\mathcal{H}_{i,t} = H(\textbf{o}_{i,t}|s_{i,t})$ is a deterministic function of the state $s_{i,t}$ alone (it depends on the policy distribution, not on the sampled token).
Therefore, $\widehat{I}(\mathcal{H}_{i,t}; r_i)$ is not an estimator of the action-level CMI $I(\textbf{o}_{i,t}; r|s_{i,t})$; it is a dataset-level diagnostic of whether entropy regimes are associated with different outcomes.
Since $I(\textbf{o}_{i,t}; r|s_{i,t}) \leq \mathcal{H}_{i,t}$ (Proposition~2), low-entropy positions rule out large CMI regardless of the observed reward.
Conversely, large $\widehat{I}$ signals that entropy differs across reward outcomes, indicating regimes where the entropy ceiling is high enough to permit non-trivial correction.
The diagnostic therefore serves as a \textit{necessary-condition detector}: it identifies positions where large hindsight shifts \emph{cannot} reside (low entropy), while flagging high-entropy positions as \emph{candidates} for non-trivial correction.
This statistic can be directly computed from the entropy distributions of positive and negative samples (Figure~\ref{fig:ent_dis}).

The entropy distribution in Figure~\ref{fig:ent_dis} shows that in the \textit{low-entropy regime}, the distributions for positive and negative tokens are nearly identical, implying $H(\mathcal{H}_{i,t} | r_i) \approx H(\mathcal{H}_{i,t})$ and thus $\widehat{I} \approx 0$.
In the \textit{high-entropy tail}, a significant divergence emerges---negative samples exhibit higher density at high entropy---yielding $\widehat{I} \gg 0$.
This empirically supports the theoretical prediction: low-entropy tokens have limited capacity for reward-conditioned information, while high-entropy tokens are candidate locations where non-trivial correction may reside.

\subsection{From Information Theory to Gradient Dynamics}
\label{app:gradient_analysis}

\noindent\textit{Note on concurrent work.}
The gradient derivation below follows a standard analysis of the softmax-parameterized policy gradient.
A similar one-dimensional (reward-polarity) gradient analysis appears concurrently in~\citet{Zhu2025TheSE} and~\citet{NegativeGradientGRPO2025}.
Our contribution is the two-dimensional extension along both reward polarity \textit{and} token entropy, which yields the four-quadrant interpretation, explains how GRPO's group-normalized advantage treats low-capacity and candidate high-capacity positions uniformly, and connects the resulting behavior to the CMI bound (Proposition~\ref{prop:credit_bound}). A summary of the key result is presented in Section~\ref{sec:gradient_analysis}.

\vspace{4pt}
The CMI bound (Proposition~2) explains \textit{why} entropy bounds possible expected hindsight shifts.
The following analysis shows \textit{how} this manifests in the gradient dynamics of GRPO: the gradient magnitude for the sampled token is modulated by the confidence term $(1-\pi_{o_t})$, which is aligned in expectation with the entropy-based capacity view.

To investigate the token-level gradient dynamics, let $z_v$ denote the logit of a token $v$ in the vocabulary $\mathcal{V}$.
The probability of generating token $v$ at step $t$, denoted as $\pi_{v} := \pi_{\theta}(v|q, o_{<t})$, is given by the softmax function:
\begin{equation}
  \pi_{v}= \frac{e^{z_v}}{\sum_{k \in \mathcal{V}} e^{z_k}}.
\end{equation}

Consider the simplified GRPO loss for a sampled token $o_t$ at step $t$.
Omitting the clipping operation and the normalization constant $\frac{1}{\sum_{i=1}^G|o_i|}$ focuses on the core gradient dynamics. The simplified objective is:
\begin{equation}
  \begin{aligned}
  \mathcal{L}_t(\theta) = - \mathbb{E} \left[ \frac{\pi_\theta(o_t|q, o_{<t})}{\pi_{\text{old}}(o_t|q, o_{<t})} \hat{A}_t \right].
  \end{aligned}
\label{eq:grpo_simplified}
\end{equation}

The gradient of this loss with respect to the logit $z_v$ can be derived using the chain rule:
\begin{equation}
  \begin{aligned}
  \frac{\partial \mathcal{L}_t}{\partial z_v}& = \frac{\partial \mathcal{L}_t}{\partial \pi_{o_t}} \cdot \frac{\partial \pi_{o_t}}{\partial z_v}.
\end{aligned}
\end{equation}

First, the derivative of the linear importance sampling ratio with respect to the probability $\pi_{o_t}$ is:
\begin{equation}
  \frac{\partial \mathcal{L}_t}{\partial \pi_{o_t}} = - \frac{\hat{A}_t}{\pi_{\text{old}}(o_t)}.
\end{equation}

Second, the derivative of the softmax function $\pi_{o_t}$ with respect to the logit $z_v$ is given by:
\begin{equation}
  \begin{aligned}
\frac{\partial \pi_{o_t}}{\partial z_v}
& = \frac{\mathbb{I}(v=o_t)e^{z_{o_t}}\sum_{k \in \mathcal{V}} e^{z_k}-e^{z_{o_t}}e^{z_{v}}}{\left(\sum_{k \in \mathcal{V}} e^{z_k}\right)^2} \\
& =  \frac{e^{z_{o_t}} \cdot \left( \mathbb{I}(v=o_t)\sum_{k \in \mathcal{V}} e^{z_k}-e^{z_{v}} \right)}{\left(\sum_{k \in \mathcal{V}} e^{z_k}\right)\cdot \left(\sum_{k \in \mathcal{V}} e^{z_k}\right)} \\
& = \pi_{o_t} (\mathbb{I}(v = o_t) - \pi_v)\\
\end{aligned}
\end{equation}
where $\mathbb{I}(v = o_t)$ is the indicator function.

Combining these terms, we obtain the gradient with respect to the logit $z_v$:
\begin{equation}
  \begin{aligned}
  \frac{\partial \mathcal{L}_t}{\partial z_v} & = - \frac{\hat{A}_t}{\pi_{\text{old}}(o_t)} \cdot \pi_{o_t} (\mathbb{I}(v = o_t) - \pi_v).
  \end{aligned}
\label{eq:gradient_logit}
\end{equation}

Under the standard assumption that the policy has not diverged significantly from the old policy (i.e., $\pi_\theta \approx \pi_{old}$), the ratio $\frac{\pi_{o_t}}{\pi_{old}(o_t)} \approx 1$. For example, $\frac{\pi_{o_t}}{\pi_{old}(o_t)}$ is between $[0.8,1.28]$ in this paper. This yields the final gradient descent form:
\begin{equation}
  \begin{aligned}
-\nabla_{z_v} \mathcal{L}_t & \approx \hat{A}_t (\mathbb{I}(v = o_t) - \pi_v) \\
& = \begin{cases}
  \hat{A}_t \cdot (1 - \pi_{o_t}), & \text{if} \ v = o_t \\
  - \hat{A}_t \cdot \pi_v, & \text{if} \ v \neq o_t
\end{cases} \\
  \end{aligned}
\label{eq:final_gradient_logit}
\end{equation}
where $\mathbb{I}(v = o_t)$ is an indicator function that equals 1 if token $v$ is the sampled token $o_t$, and 0 otherwise.
Suppose a learning rate of $\eta$, the logit update for token $v$ is:
\begin{equation}
  z_v  \leftarrow z_v - \eta \cdot \nabla_{z_v} \mathcal{L}_t 
  \label{eq:z_update}
\end{equation}

This derivation reveals how sampled-token probability, which is correlated with entropy in expectation, modulates the gradient update:

\noindent\textbf{High-Entropy Quadrants (PHR \& NHR).} 
In expectation, high-entropy regimes tend to place lower probability on sampled tokens than near-deterministic regimes. Consequently, the scaling factor $(1 - \pi_{o_t})$ tends to be larger on average, increasing the expected gradient magnitude.
\begin{itemize}[leftmargin=*]
    \item \textbf{PHR ($\hat{A}_t > 0$):} The model receives a positive-polarity update to increase the probability of the sampled token $o_t$ while decreasing the logits of all other tokens following Equation~\ref{eq:z_update}. This update can consolidate a correct branch that was previously uncertain, sharpening the local output distribution. This helps explain why PHR improves both accuracy ($\text{Avg@}k$) and diversity ($\text{Pass@}k$), as shown in Figure~\ref{fig:decompose} and~\ref{fig:TT_performance}.
    \item \textbf{NHR ($\hat{A}_t < 0$):} The model receives a negative-polarity update. For the sampled token $o_t$, the logit is decreased. Crucially, for unsampled tokens $v \neq o_t$, the gradient update is positive because $- \nabla_{z_v}\mathcal{L}_t \propto - \hat{A}_t \pi_v$ (where $\hat{A}_t$ is negative). This increases the logits of unsampled tokens $z_v$ in Equation~\ref{eq:z_update}, redistributing probability mass to other plausible candidates and pruning erroneous branches. Figure~\ref{fig:decompose} shows that NHR not only encourages diversity (Pass@$32$) but also improves accuracy (Avg@$32$), as the model becomes less likely to diverge into incorrect reasoning paths.
\end{itemize}

\noindent\textbf{Low-Entropy Quadrants (PLR \& NLR).} 
Low entropy implies a peaked distribution where the model is confident in its choice ($\pi_{o_t} \approx 1$). Thus, the scaling gradient factor $(1 - \pi_{o_t})$ is relatively small.
\begin{itemize}[leftmargin=*]
    \item \textbf{PLR ($\hat{A}_t > 0$):} Since the model is already confident and correct, the output distribution is sharp. Further reinforcement has limited room to induce reward-conditioned redistribution and can instead encourage overconfidence or entropy collapse. This explains the stagnation of PLR performance observed in our experiments.
    \item \textbf{NLR ($\hat{A}_t < 0$):} Although the gradient magnitude is small, any non-zero update penalizes a token the model is confident in. In negative trajectories, low-entropy tokens often represent syntactically correct structures or robust logical transitions (e.g., "The answer is") rather than the root cause of the error. Penalizing these tokens with a negative trajectory-level advantage introduces noise, potentially degrading the model's linguistic capabilities without correcting the actual reasoning flaw. This accounts for why NLR is less effective than NHR in Figure~\ref{fig:decompose}.
\end{itemize}

\subsection{Connections to Related Theoretical Frameworks}
\label{app:theory_connections}

The CMI framework developed above offers a common capacity-based lens for interpreting several recent empirical findings.

\paragraph{Negative Reinforcement Effectiveness.}
\citet{Zhu2025TheSE} demonstrated the ``surprising effectiveness'' of negative reinforcement in LLM reasoning, showing that high-entropy negative-polarity updates can deliver strong learning effects.
The CMI framework offers a capacity-based interpretation: at high-entropy positions where $\mathcal{H}_{i,t}$ is large, the bound $I(\textbf{o}_{i,t}; r | s_{i,t}) \leq \mathcal{H}_{i,t}$ permits non-trivial reward-conditioned correction.
Negative feedback at these positions (NHR) can prune erroneous branches by redistributing probability mass away from incorrect tokens in the observed interventions, as shown in the gradient analysis (Equation~\ref{eq:final_gradient_logit}).
Conversely, negative feedback at low-entropy positions (NLR) targets tokens with limited capacity under the CMI bound, explaining why it can introduce noise rather than useful updates.
Recent analysis of negative gradient dynamics~\cite{NegativeGradientGRPO2025} further corroborates this view, showing that the gradient contribution of negative samples is concentrated on tokens where the policy is most uncertain---matching the high-capacity regime identified by Proposition~2.

\paragraph{Token Hidden Reward and Implicit Token Scoring.}
The Token Hidden Reward (THR) framework~\cite{TokenHiddenReward2025} identifies ``reward-carrying'' tokens through an importance-based scoring mechanism.
Remarkably, the set of tokens identified as high-THR overlaps with high-entropy tokens by over 90\%~\cite{TokenHiddenReward2025}.
Our CMI analysis contextualizes this empirical coincidence: because the CMI bound $I(\textbf{o}_{i,t}; r | s_{i,t}) \leq \mathcal{H}_{i,t}$ rules out large expected shifts at low-entropy positions, high-entropy positions remain candidates \textit{capable} of supporting substantial reward-conditioned information.
THR's learned scoring function implicitly discovers positions that often overlap with high-entropy regions; \modelname uses a training-free entropy-based proxy to target the same candidate regime without the overhead of learning an auxiliary reward model.

\paragraph{Learning Dynamics and Entropy Collapse.}
Studies on the learning dynamics of LLM fine-tuning~\cite{LearningDynamicsLLM2025} have identified a ``squeezing effect'' whereby PLR training progressively collapses the entropy of already-confident tokens, ultimately reducing the model's exploration capacity.
This phenomenon is consistent with the gradient structure in the low-entropy regime: as $\pi_{o_t} \to 1$, the gradient factor $(1 - \pi_{o_t}) \to 0$, creating a self-reinforcing pattern where confident tokens receive diminishing updates but remain in a peaked distribution.
From the CMI perspective, PLR operates in a regime where $\mathcal{H}_{i,t} \approx 0$, so the CMI upper bound implies near-zero capacity for reward-conditioned shifts regardless of the applied advantage magnitude.
This helps explain why PLR exhibits stagnating performance in our four-quadrant experiments (Figure~\ref{fig:decompose}): optimization effort is spent on tokens that have limited capacity for hindsight correction.

\begin{figure}[!h]
    \centering
    \includegraphics[width=\linewidth]{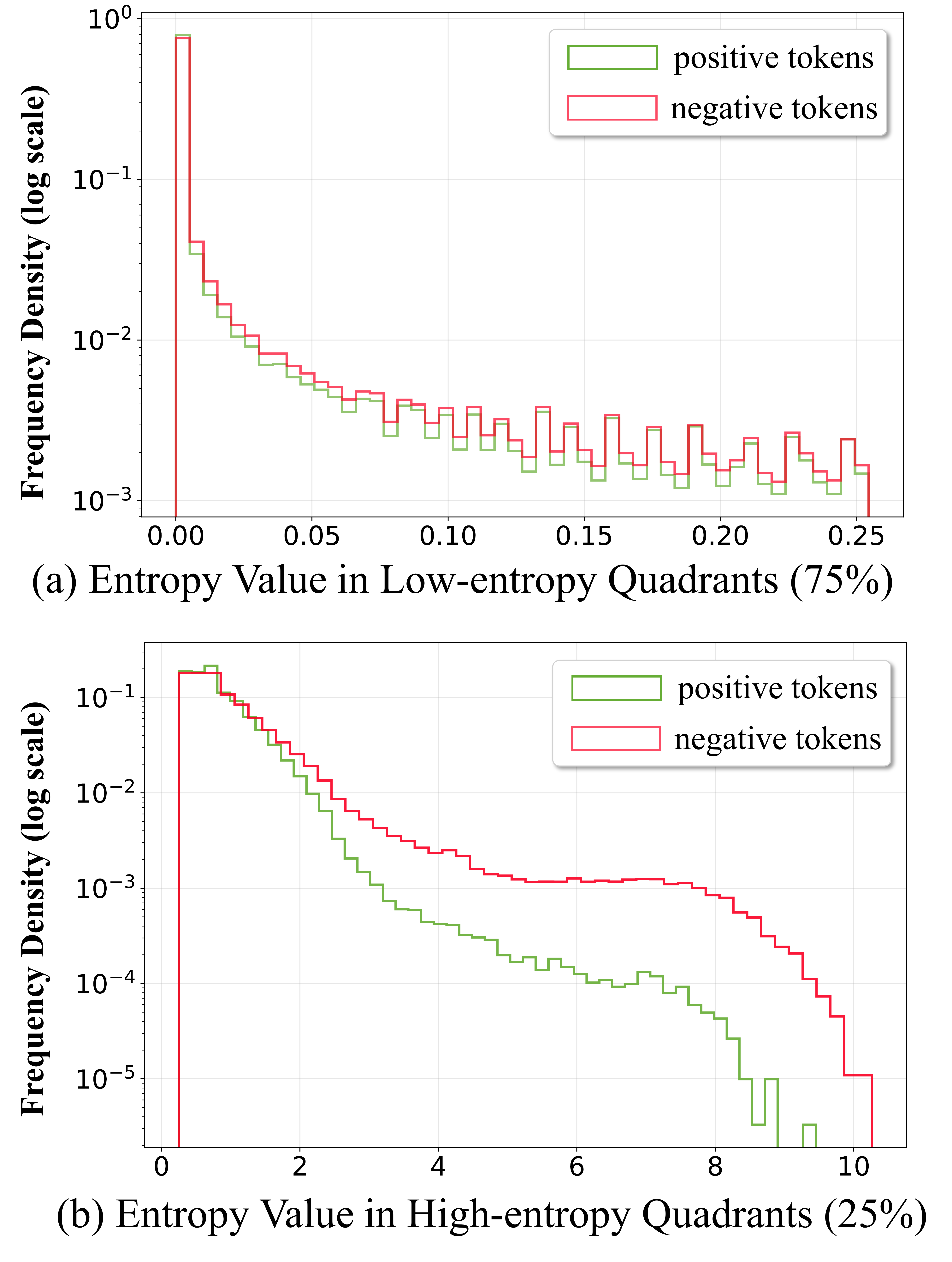}
    \caption{Token entropy distribution. Low-entropy distributions remain aligned, but a sharp divergence emerges between positive and negative samples in the high-entropy quadrant. }
    
    \label{fig:ent_dis}
\end{figure}

\begin{figure*}[!t]
  \centering
  \includegraphics[width=\linewidth]{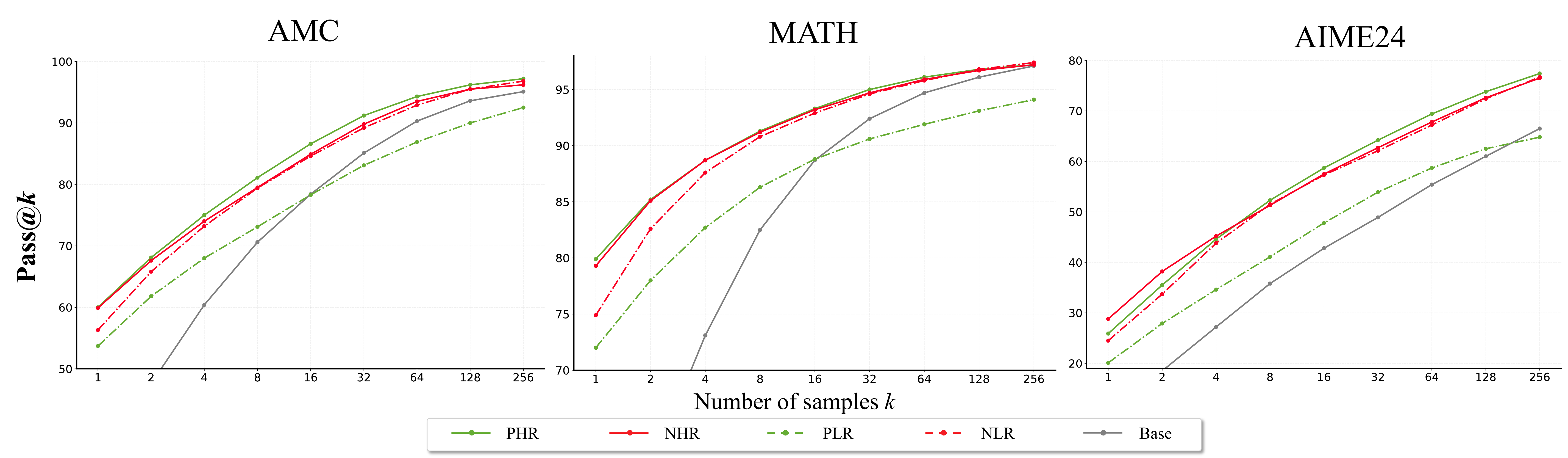}
  \caption{Test time compute performance of four-quadrant training objectives. The x-axis represents varying $k \in \{1, 2, 4, 8, 16, 32, 64, 128, 256\}$ values, covering the full spectrum of $\text{Pass@}k$. }
  \label{fig:TT_performance}
\end{figure*}

\section{Decomposition Training Results}
\label{app:decompose_exp}
\subsection{Training Entropy Distribution}

Figure~\ref{fig:ent_dis} visualizes the entropy distribution of tokens in positive and negative training rollouts.
Only a minority of tokens (approximately 25\%) exhibit high entropy.
In the low-entropy regime, the distributions for positive and negative tokens are nearly identical.
This overlap indicates that low-entropy tokens typically correspond to deterministic reasoning steps (e.g., syntax or factual recall) that remain stable regardless of the final outcome.
However, a significant divergence emerges in the high-entropy tail, where the density of negative tokens is consistently higher than that of positive tokens.
This pattern points to high-uncertainty states being more often associated with divergent or erroneous reasoning paths.
These empirical findings support the targeted, signed proxy strategy used by \modelname: rather than assigning the same update to all tokens in a trajectory, the objective consolidates successful high-entropy branches and redirects failed high-entropy branches.

\subsection{Test-time Compute Performance}
\label{app:tt_performance}
Evaluation uses the final optimization checkpoints for PHR, NHR, PLR, and NLR across varying $\text{Pass@}k$ values with decoding temperature set $1.0$.
In summary, PHR and NHR achieve robust performance across the full range of $k$.
The test-time compute results on AMC, MATH, and AIME24 (Figure~\ref{fig:TT_performance}) support the following observations:
\begin{itemize}[leftmargin=*, itemsep=0pt]
  \item Across the $\text{Pass@}k$ spectrum ($k=1$ to $256$), PHR and NHR generally outperform both the Base model and the low-entropy variants. This supports the view that signed updates at high-uncertainty positions are the main observed locus of sustained reasoning improvements in these diagnostics.
  \item PLR shows initial gains over the Base model at low $k$ values (e.g., $k=1, 2$), but its performance scaling is significantly flatter. In AIME24, the Base model eventually surpasses PLR as $k$ increases. This trend is even more pronounced in less challenging datasets such as AMC and MATH, indicating that reinforcing deterministic, low-entropy tokens leads to over-exploitation and reduces output diversity.
  \item In line with~\citet{Zhu2025TheSE}, NHR alone is remarkably effective: it substantially outperforms NLR across the full $\text{Pass@}k$ spectrum and approaches the performance of the complete GRPO baseline at high $k$. This pattern indicates that penalizing incorrect high-entropy tokens can prune erroneous reasoning branches and redistribute probability mass toward plausible alternatives, providing a more targeted update than penalizing low-entropy tokens.
  \item NLR-only training yields lower scores at low $\text{Pass@}k$ (e.g., $k=1, 2$) compared to NHR. This indicates that penalizing low-entropy tokens in incorrect responses provides a noisy or redundant update. These tokens (e.g., mathematical syntax such as \texttt{\textbackslash frac} or common suffixes) are typically deterministic and correct even in failed trajectories. Penalizing them risks destabilizing the model's linguistic capabilities without correcting the underlying logical errors.
\end{itemize}
\section{Additional Analyses}

\subsection{Entropy Dynamics of \modelname}
\label{app:entropy_dynamics}

\begin{figure}[h]
  \centering
    \includegraphics[width=\linewidth]{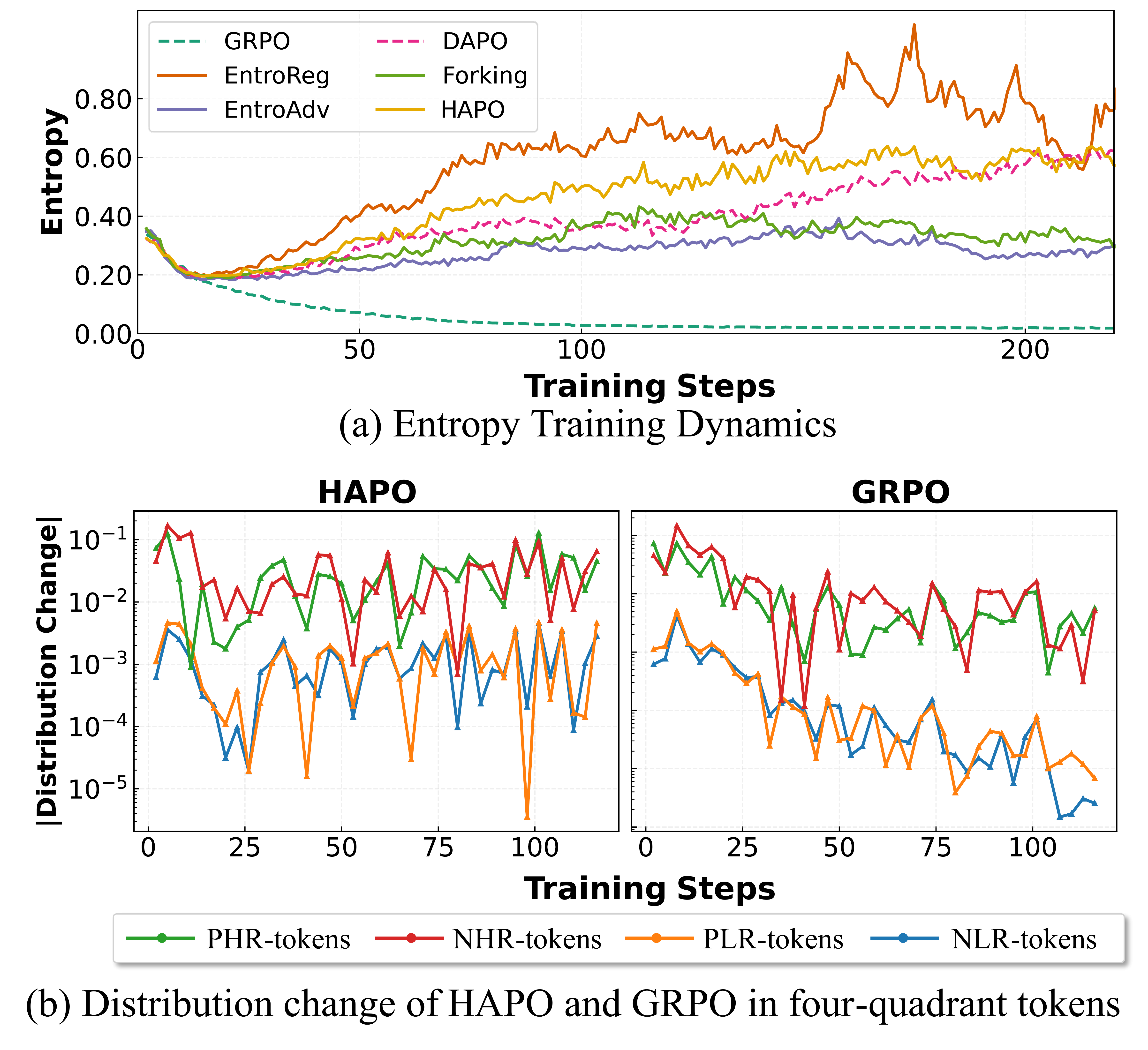}
    \vspace{-0.8cm}
  \caption{Entropy dynamics analysis of 7B models.}
  \label{fig:ent_compare}
\end{figure}

Figure~\ref{fig:ent_compare} shows that \modelname maintains a stable entropy trajectory throughout training.
In contrast, GRPO suffers from entropy collapse, which constrains exploratory capacity, while EntroReg tends to over-flatten the output distribution.
The analysis also tracks absolute entropy shifts across the four quadrants.
The largest observed entropy shifts occur in the high-entropy quadrants (PHR and NHR), with low-entropy quadrants remaining comparatively stable. This pattern accords with \modelname concentrating larger update effects on uncertain segments while preserving confidence in established predictions.

\subsection{Efficiency Analysis}
\label{app:efficiency_additional}

If \modelname indeed concentrates updates on consequential decision points, it should reach a better accuracy--cost frontier than methods that rely on uniform advantage assignment or longer generations.
Figure~\ref{fig:effi_compare} in the main text gives the AIME24 view: \modelname converges more smoothly than GRPO and DAPO and avoids the over-generation pattern of EntroAdv, yielding a better accuracy--cost trade-off.
Table~\ref{tab:efficiency_additional} extends the same comparison to MATH and Olympiad.
Across both benchmarks, \modelname achieves the highest accuracy among all methods.
Compared to GRPO, DAPO, and Forking, it attains meaningfully higher Avg@4 at a moderate increase in generation length; compared to EntroAdv, it reaches comparable or better accuracy with fewer tokens, indicating a more favorable accuracy--cost trade-off.

\begin{table*}[t]
\centering
\setlength{\tabcolsep}{5pt}
\begin{tabular}{lcccc}
\toprule
\textbf{Method} & \textbf{MATH Avg@4} & \textbf{MATH Tokens} & \textbf{Olympiad Avg@4} & \textbf{Olympiad Tokens} \\
\midrule
GRPO & 88.9 & 3172 & 58.7 & 5721 \\
DAPO & 89.9 & 2690 & 58.4 & 5113 \\
Forking & 87.6 & 2769 & 59.1 & 5193 \\
KL-Cov & 90.0 & 3396 & 59.9 & 6629 \\
EntroAdv & 90.2 & 3963 & 59.1 & 7752 \\
\modelname & \textbf{91.1} & 3569 & \textbf{60.8} & 7353 \\
\bottomrule
\end{tabular}
\caption{Extended efficiency evidence on MATH and Olympiad for \texttt{DeepScaleR-1.5B-8k}. Higher Avg@4 is better; lower token counts indicate lower inference cost. The table is derived from the additional rebuttal measurements used to corroborate the AIME24 Pareto analysis.}
\label{tab:efficiency_additional}
\end{table*}

\subsection{Impact of Coefficient \texorpdfstring{$\alpha$}{alpha}}
\label{app:hyperparameter_full}

\begin{figure*}[!ht]
  \centering
  \includegraphics[width=\linewidth]{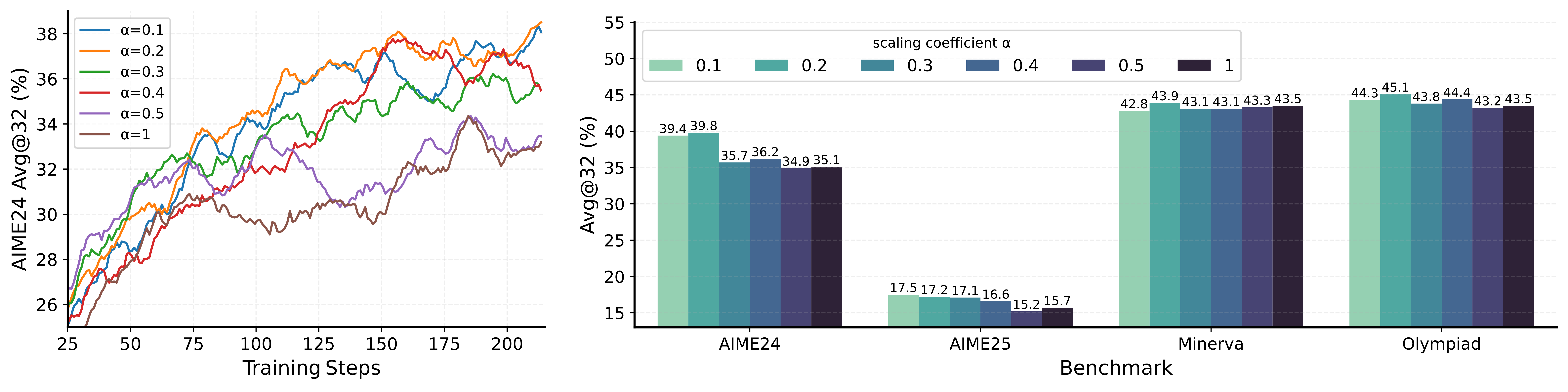}
  \caption{Hyperparameter analysis for $\alpha \in \{0.1, 0.2, 0.3, 0.4, 0.5, 1\}$ on 7B models. \textit{Left}: training efficiency comparison. \textit{Right}: evaluation on various benchmarks.}
  \label{fig:hyper_combine}
\end{figure*}

Table~\ref{tab:full_para} examines the sensitivity of \modelname to the hyperparameter $\alpha$.
Figure~\ref{fig:hyper_combine} shows that training remains robust for $\alpha \leq 0.4$; excessively large values slow convergence and degrade benchmark performance.
Across both the trajectory plots and the benchmark summary, $\alpha=0.2$ gives the strongest overall trade-off among the tested values, so it is the default setting in the main experiments. The clipping range remains fixed at $\varphi=2$ throughout these experiments, and a full $\varphi$ sweep is left to future work.
\begin{table}[h]
\resizebox{1\linewidth}{!}{
\begin{tabular}{lccccccc}
\toprule
$\alpha$ & AIME24 & AIME25 & AMC  & MATH & Minerva & Olympiad & avg  \\ \hline
1      & 35.1   & 15.7   & 63.5 & 84.3 & 43.5    & 43.5     & 47.6  \\
0.5    & 34.9   & 15.2   & 62.9 & 83.5 & 43.3    & 43.2     & 47.1  \\
0.4    & 36.2   & 16.6   & 62.7 & 83.8 & 43.1    & 44.4     & 47.8  \\
0.3    & 35.7   & 17.1   & 61.8 & 82.8 & 43.1    & 43.8     & 47.3  \\
0.2    & 39.8   & 17.2   & 62.1 & 83.7 & 43.9    & 45.1     & \textbf{48.6}  \\
0.1    & 39.4   & 17.5   & 61.2 & 83.1 & 42.8    & 44.3     & \underline{48.0}\\
\bottomrule
\end{tabular}
}
\caption{$\text{Avg@}k$ results on different benchmarks with 7B model. \textbf{Bold} and \underline{underline} numbers denote the best and second-best results.}
\label{tab:full_para}
\end{table}

\section{MMLU-Pro Results}
\label{app:mmlu_pro}

Table~\ref{tab:mmlu_pro} presents the per-domain results on MMLU-Pro. Among the compared models, \modelname achieves the highest overall accuracy (36.6\%), with particularly strong gains in reasoning-intensive domains such as economics (+7.7\% over GRPO), biology (+11.4\%), and psychology (+5.1\%).
Domains where performance is comparable to GRPO (e.g., health, chemistry) may rely more on factual retrieval than multi-step reasoning, in line with the possibility that the capacity-guided mechanism primarily benefits high-entropy decision points.

\begin{table}[h]
\centering
\resizebox{1\linewidth}{!}{
\begin{tabular}{lccc}
\toprule
\textbf{Domain} & \textbf{Base} & \textbf{GRPO} & \textbf{\modelname} \\
\midrule
Overall (Mean) & 18.2 & 33.1 & \textbf{36.6} \\
\midrule
Computer Science & 31.0 & 41.0 & \textbf{43.2} \\
Math & 32.9 & 34.1 & \textbf{37.9} \\
Chemistry & 28.9 & \textbf{30.4} & 29.2 \\
Engineering & 16.4 & \textbf{29.8} & \textbf{29.8} \\
Law & 7.1 & 12.8 & \textbf{16.2} \\
Biology & 21.6 & 44.1 & \textbf{55.5} \\
Health & 10.6 & \textbf{25.2} & 24.7 \\
Physics & 21.7 & 36.7 & \textbf{41.4} \\
Business & 24.0 & 47.9 & \textbf{50.1} \\
Philosophy & 13.0 & 25.7 & \textbf{31.3} \\
Economics & 7.4 & 48.2 & \textbf{55.9} \\
Other & 11.8 & 29.8 & \textbf{32.6} \\
Psychology & 8.4 & 40.5 & \textbf{45.6} \\
History & 10.8 & 16.8 & \textbf{24.2} \\
\bottomrule
\end{tabular}
}
\caption{Per-domain accuracy (\%) on MMLU-Pro with \texttt{Qwen2.5-Math-7B}. Best and tied-best results are in \textbf{bold}.}
\label{tab:mmlu_pro}
\end{table}

\section{Advantage Allocation Analysis}
\label{app:credit_assignment_analysis}
\begin{figure}[h!]
  \centering
  \includegraphics[width=\linewidth]{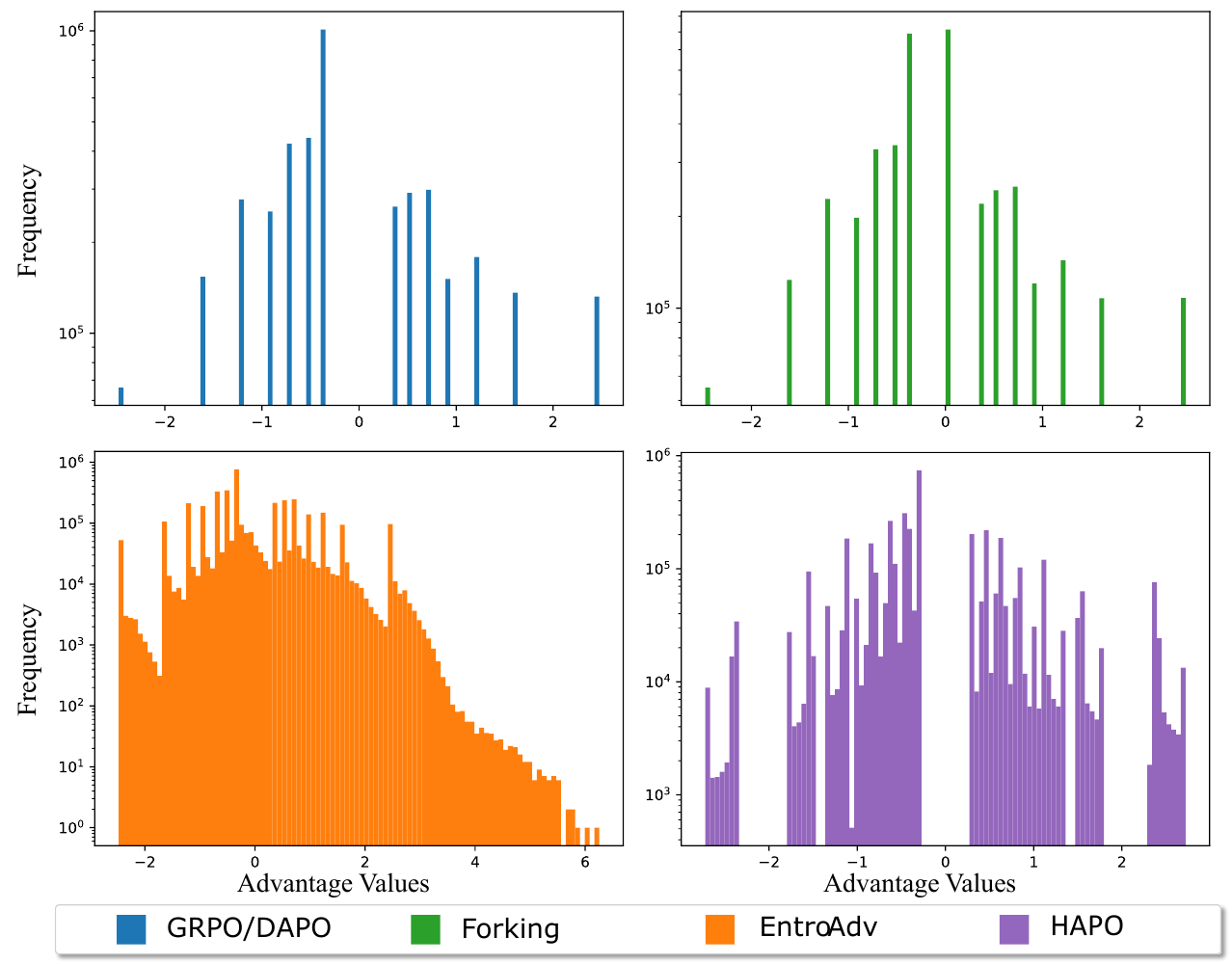}
  \caption{Advantage value distributions. Comparison of token-level advantage allocation among different methods.}
  \label{fig:credit_assignment_comparison}
\end{figure}

This section compares our advantage allocation strategy with three strong baselines: DAPO~\cite{DBLP:journals/corr/abs-2503-14476}, EntroAdv~\cite{DBLP:journals/corr/abs-2506-14758}, and Forking~\cite{DBLP:journals/corr/abs-2506-01939}.
Figure~\ref{fig:credit_assignment_comparison} visualizes the token-level advantage distributions assigned by each method.
Briefly, DAPO (and GRPO) adopts a group-based reward assignment that treats all tokens within a trajectory equally, disregarding their local uncertainty and capacity differences.
EntroAdv modifies the advantage function by adding a gradient-detached entropy term, explicitly encouraging exploration at high-uncertainty positions.
Forking implements a hard selection mechanism, restricting policy gradient updates to the top 20\% of tokens with the highest entropy, under the assumption that these forking tokens are useful high-uncertainty candidates.

Figure~\ref{fig:credit_assignment_comparison} shows that GRPO/DAPO assigns uniform advantages, failing to distinguish candidate high-capacity branching points from routine tokens.
Forking improves upon this by filtering out low-entropy tokens, yet its binary masking remains a coarse-grained approximation.
EntroAdv introduces continuous entropy-based shaping, but its additive entropy bonus does not distinguish reward polarity and lacks a distinct mechanism for handling negative feedback.
In contrast, our method implements a continuous, four-quadrant modulation.
It modulates candidate high-capacity tokens by increasing positive feedback at successful branching points (PHR) and increasing negative feedback at failed branching points (NHR), while dampening the signal for low-entropy tokens to prevent over-exploitation and noise.
This targeted approach concentrates update magnitude on branching points rather than routine continuations.

\section{Empirical Observation}
\label{app:example_output}
This section examines the responses of \texttt{Qwen2.5-Math-7B} by visualizing the training outputs with token entropy highlighted.
\begin{figure}[!th]
  \centering
  \includegraphics[width=\linewidth]{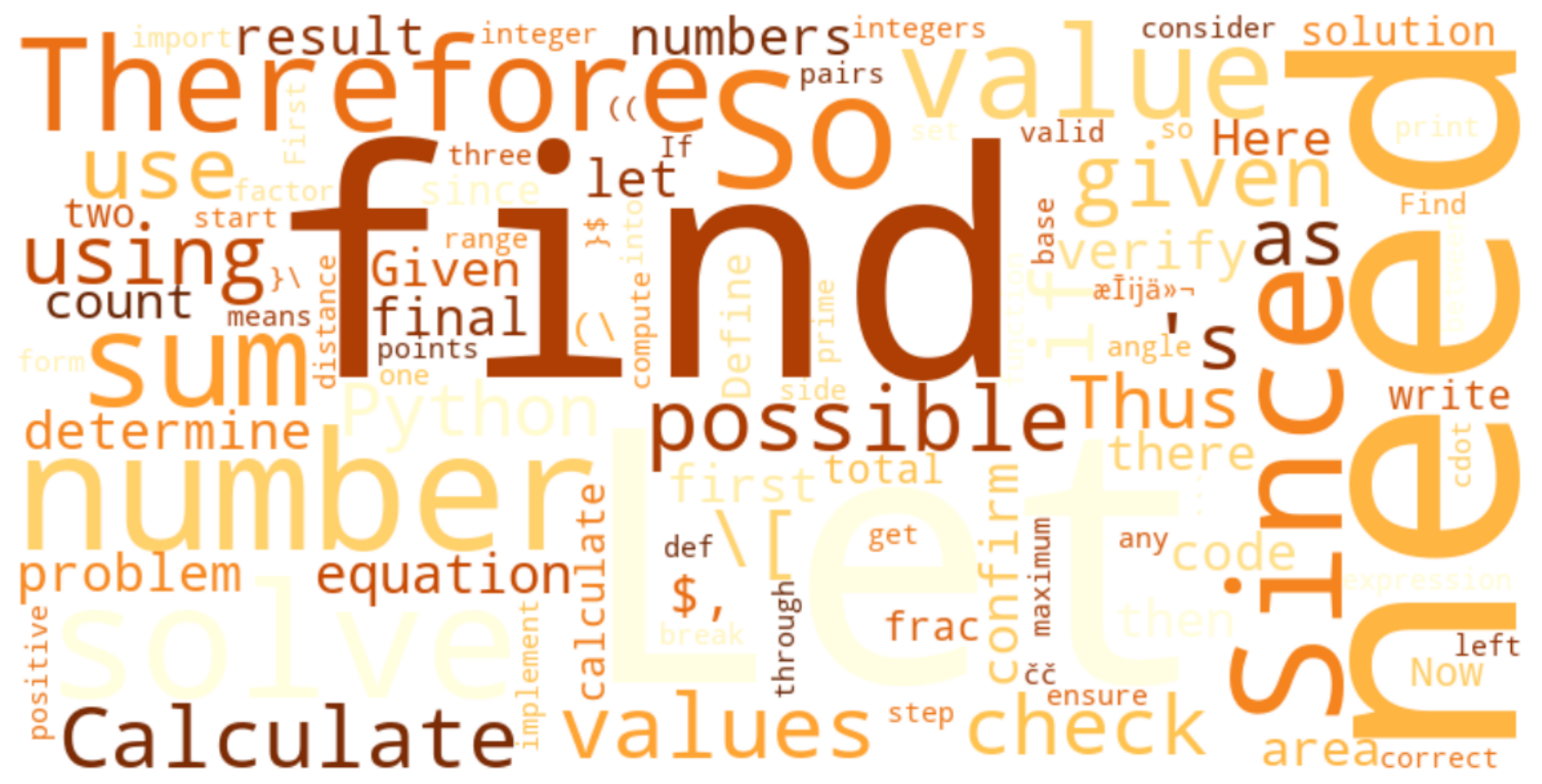}
  \caption{High-entropy token word cloud.}
  \label{fig:high-ent_wordcloud}
  \includegraphics[width=\linewidth]{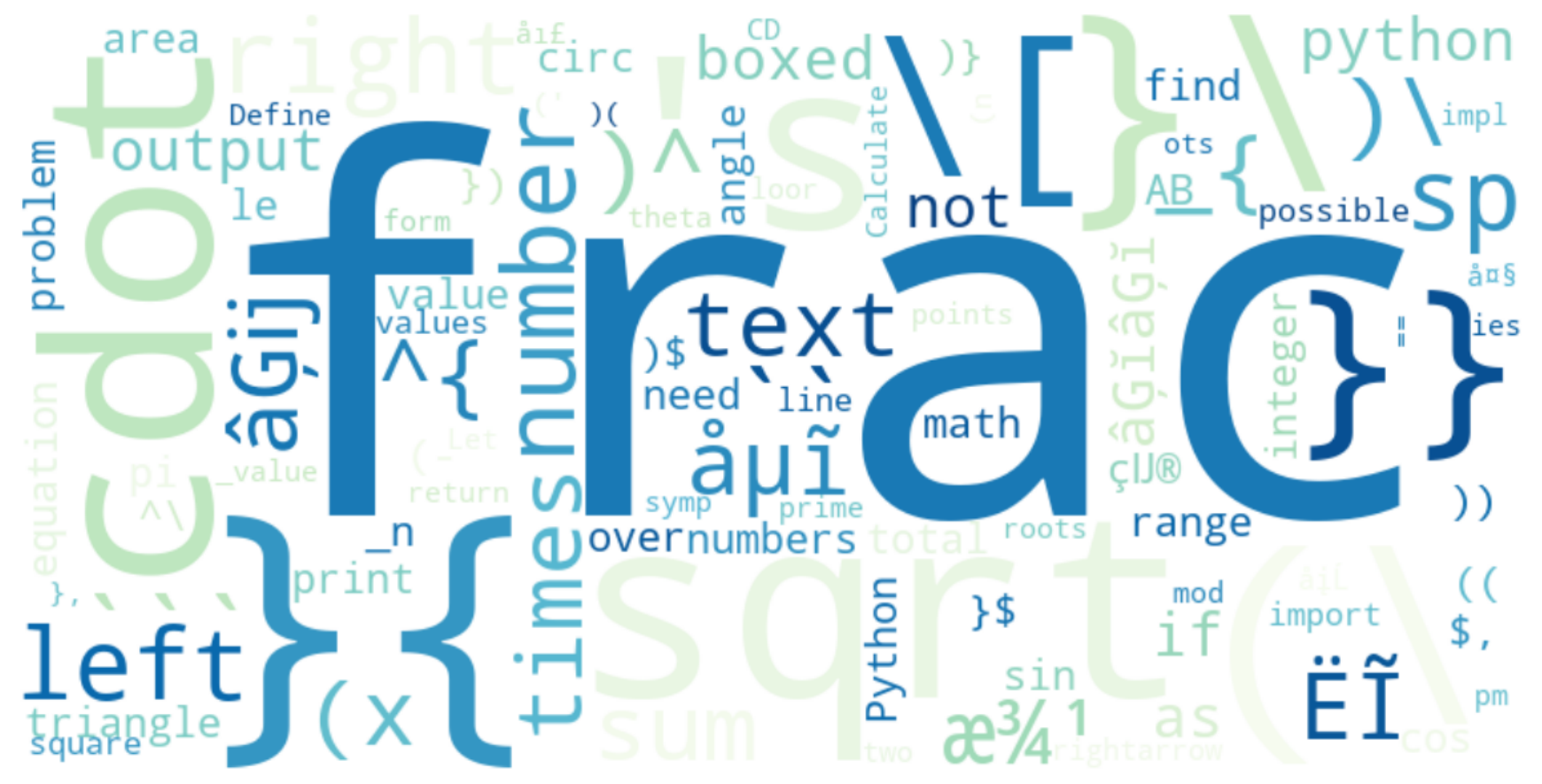}
  \caption{Low-entropy token word cloud.}
  \label{fig:low-ent_wordcloud}
\end{figure}

The visualization samples over 4k responses in the training set and collects over $10^6$ tokens.
The statistical analysis of the entropy values of tokens is presented in word clouds in Figure~\ref{fig:high-ent_wordcloud} and~\ref{fig:low-ent_wordcloud}, where a larger font size indicates a higher average token entropy.
The visualization aligns with prior works~\cite{DBLP:journals/corr/abs-2506-01939,DBLP:journals/corr/abs-2506-14758}: high-entropy tokens often coincide with candidate forks in reasoning trajectories, while low-entropy tokens often facilitate the stable execution of logic along the selected path.
Specifically, high-entropy tokens frequently correspond to logical connectors such as "since," "therefore," and "thus," which bridge reasoning steps. Similarly, in mathematical contexts, tokens like "need," "given," and "let" serve to introduce assumptions or definitions. By contrast, low-entropy tokens typically include word suffixes, code snippets, or components of mathematical formulas, all of which are characterized by high structural determinism.
However, existing works primarily target a binary classification of tokens into high and low entropy categories.
A further question is whether entropy levels correspond to consistent model behaviors, regardless of correctness.
For example, do low-entropy tokens usually indicate overconfident mistakes or robust reasoning steps in incorrect responses?

To this end, the analysis visualizes two representative examples in Figure~\ref{fig:example_output}, where token entropy is highlighted separately in incorrect and correct reasoning.
The visualization reveals similar patterns.
Despite the final correctness, high-entropy tokens often coincide with candidate decision points along the reasoning trajectory, while low-entropy tokens often execute deterministic steps along the established path.
Although some low-entropy tokens exhibit overconfident hallucinations, the overall trend persists~\cite{DBLP:conf/iclr/0026L0GWTFY24,DBLP:conf/iclr/XiongHLLFHH24}.

\begin{figure*}[hp]
  \centering
  \includegraphics[width=\linewidth]{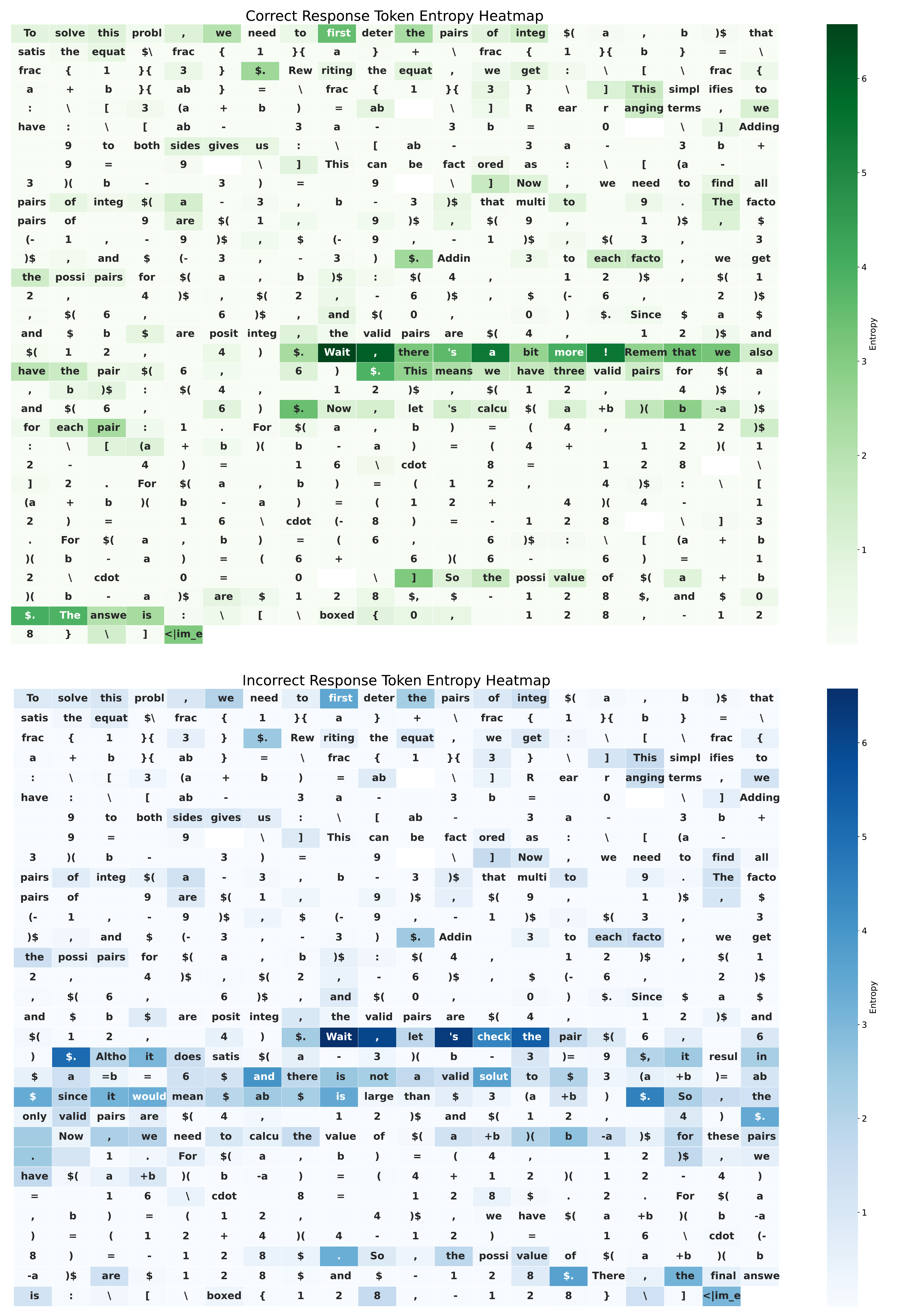}
  \caption{Visualization of token entropy in incorrect and correct responses.}
  \label{fig:example_output}
\end{figure*}

\end{document}